\documentclass[10pt]{article}

\pdfoutput=1

\usepackage[normalem]{ulem}

\usepackage[latin1]{inputenc}
\usepackage{mystyle}
\usepackage{dsfont}
\usepackage{url}
\usepackage{fullpage}

\title{Degrees of freedom for off-the-grid sparse estimation}

\newcommand{\lasso}{Lasso}
\newcommand{\blasso}{Blasso}

\newcommand{\bell}{\boldsymbol{\ell}}

\author{%
Clarice Poon\footnote{Department of Mathematical Sciences, University of Bath, Bath BA2 7AY, UK, \texttt{cmshp20@bath.ac.uk}}, \quad%
Gabriel Peyr\'e\footnote{CNRS and DMA, Ecole Normale Sup\'erieure, 45 rue d'Ulm, F-75230 PARIS cedex 05, FRANCE, \texttt{gabriel.peyre@ens.fr} }%
}
\date{\today}

\begin{document}

\maketitle


\begin{abstract}
A central question in modern machine learning and imaging sciences is to quantify the number of effective parameters of vastly over-parameterized models. 
The degrees of freedom is a mathematically convenient way to define this number of parameters. Its computation and properties are well understood when dealing with discretized linear models, possibly regularized using sparsity. In this paper, we argue that this way of thinking is plagued when dealing with models having very large parameter spaces. In this case it makes more sense to consider ``off-the-grid'' approaches, using a continuous parameter space. This type of approach is the one favoured when training multi-layer perceptrons, and is also becoming popular to solve super-resolution problems in imaging. 
Training these off-the-grid models with a sparsity inducing prior can be achieved by solving a convex optimization problem over the space of measures, which is often called the Beurling \lasso~(\blasso), and is the continuous counterpart of the celebrated \lasso~parameter selection method. 
In previous works \cite{zou2007degrees,dossal2013degrees}, the degrees of freedom for the Lasso was shown to coincide with the size of the smallest solution support.
Our main contribution is a  proof of a continuous counterpart to this result for the \blasso. While in dimension $d$, each of the $k$ non-zero recovered atom in the recovered measure carries over $d+1$ parameters ($d$ for the position and 1 for the weight), a surprising implication of our new formula it that the degrees of freedom for these off-the-grid models is in general strictly smaller $(d+1)k$. Our findings thus suggest that discretized methods actually vastly over-estimate the number of intrinsic continuous degrees of freedom. 
Our second contribution is a detailed study of the case of sampling Fourier coefficients in 1D, which corresponds to a super-resolution problem. 
We show that our formula for the degrees of freedom is valid outside of a set of measure zero of observations, which in turn justifies its use to compute an unbiased estimator of the prediction risk using the Stein Unbiased Risk Estimator (SURE).
We also report numerical results for both the case of Fourier sampling and the learning of a multilayers perceptron with a single hidden layer. These experimental simulations highlight the strong bias induced by discretization errors, which makes the \lasso~approach inconsistent to approximate the risk of continuous models.  
\end{abstract}

\newcommand{\dof}{\mathrm{dof}}
\newcommand{\dive}{\mathrm{div}}

\newcommand{\mes}{\mathbf{m}}

\section{Introduction}

\subsection{Stein's lemma and degrees of freedom}

Given a Gaussian random variable $y \sim \Nn(\mu, \sigma^2 \Id_n)$ with mean $\mu \in \RR^n$ and standard deviation $\sigma>0$ and a weakly differentiable estimator of $\mu$, $\hat \mu : \RR^n 
\to \RR^n$, the degrees of freedom (dof) of the estimator  is defined to be
\begin{equation}
	\dof(\hat \mu) \eqdef \frac{1}{\sigma^2}\sum_i \mathrm{Cov}(y_i, \hat \mu_i(y)) =  \frac{1}{\sigma^2}\sum_i \EE[(y_i - \mu_i)\hat \mu_i(y)].
\end{equation}
A-priori, computation of this requires knowledge of the unknown $\mu$. However, a remarkable result of Stein \cite{stein1981estimation} shows that
$$
	\dof(\hat \mu) =\EE_y[ \dive(\hat \mu)(y)], 
$$
where $\dive(\hat \mu)(y) \eqdef \sum_i \frac{\partial \hat \mu_i}{\partial y_i}(y)$. Therefore, $\dive(\hat \mu)$ is an unbiased estimate of the degrees of freedom and requires only the divergence of $\hat \mu$ evaluated on the data. 
In the particular case where $\hat\mu$ is an orthogonal projector on some affine space,  $\dof(\hat \mu)$ is equal to the dimension of the space. 

Estimation of the degrees of freedom using $\dive(\hat \mu)$ plays a pivotal role in the definition of parameter selection procedures (typically to select an optimal regularization parameter, denoted $\la$ below) using various penalized empirical risk procedures, such as    Mallows' $C_p$~\cite{mallows1973some}, Akaike information criterion (AIC)~\cite{akaike1998information}, Bayesian information criterion (BIC)~\cite{schwarz1978estimating}, Generalized Cross-Validation (GCV)~\cite{golub1979generalized}.
In the specific case of Gaussian random vectors, one can even have access to an unbiased estimator of the risk using the Stein Unbiased Risk Estimator (SURE)~\cite{stein1981estimation}, since one has
\eq{
	\EE( \norm{\hat \mu(y)-\mu}^2 ) = \EE( \sure(\hat \mu)(y) )
	\qwhereq
	\sure(\hat \mu)(y) \eqdef -n \sigma^2 + \norm{y - \hat \mu}^2 +  2\sigma^2 \dive(\hat \mu)(y).
}
Note that the quantity $\sure(\hat \mu)(y)$ does not involve $\mu$, and can thus in practice be estimated from a realization of the observation $y$ alone. We refer to Section~\ref{sec-sure} for more details about the SURE. 
The use of degrees of freedom beyond Gaussian vectors, and in particular for exponential families, is studied for instance in~\cite{hudson1978nie,Hwang82,eldar-gsure}. It  is thus possible to use our results in these extended settings. 

Since the degrees of freedom plays an important role in risk estimation and parameter selection, it is pertinent to understand how to compute the divergence of estimators. 
For simple shrinkage operators, $\hat\mu$ and thus $\dive(\hat \mu)$ can be computed in closed form, and the corresponding SURE method is thus extensively used for denoising~\cite{donoho1995adapting}.
The last decades have seen the introduction of various non-linear estimators, and in particular methods based on penalized optimization procedures using sparsity-enforcing priors (such as the Lasso), which we detail next.  
For such estimators, typically computed approximately using an iterative scheme, the computation of $\dive(\hat \mu)$ can be implemented efficiently and stably using Monte-Carlo sampling~\cite{ramani2008monte} and recursive automatic differentiation~\cite{giryes2011projected,ramani2012regularization,deledalle2014stein}. In some cases (highlighted below), one can also give a mathematical expression of $\dive(\hat \mu)$ involving the solution of the optimization problem. 
The purpose of this paper is to achieve a similar theoretical understanding of the degrees of freedom for the so-called Beurling-\lasso~estimator, which is an infinite-dimensional version of the Lasso. One challenging aspect of this estimator is that it requires the resolution of an optimization problem over an infinite dimensional space (of Radon measures), and unlike previous works, the set of all possible recovered support/parameters cannot be countably enumerated, thus making  existing proof strategies ineffective (which are inherently finite dimensional).

\subsection{DOF of the \lasso}

Given $y\in \RR^n$ and a design matrix $X\in \RR^{n\times p}$, the \lasso~estimate is
\begin{equation}\label{eq-lasso}
	\hat \beta \in \uargmin{\beta\in \RR^p} \frac{1}{2}\norm{y - X \beta}^2_2 + \lambda\norm{\beta}_1.
\end{equation}
We assume that $y\in \Nn(\mu, \sigma^2 \Id_n)$ for some $\mu\in\RR^n$, $\sigma>0$ and we consider the estimator $\hat \mu(y) = X \hat \beta$. Note that by strong convexity of the $\ell_2$ term, $\hat \mu(y)$ is uniquely defined, and, even if $\hat \beta$ is not unique, the value  of $\hat \mu(y)$ is independent of the specific choice of a solution $\hat\beta$. 

The degrees of freedom for the \lasso~estimator has been studied in~\cite{zou2007degrees} for injective design matrices $X$ and~\cite{tibshirani2012degrees,dossal2013degrees} for arbitrary matrices. These works show that $\hat\mu$ is Lipschitz and hence differentiable almost everywhere (outside an explicit set of affine hyperspaces), and moreover, for almost every $y\in \RR^n$, the degrees of freedom can be expressed in terms of the smallest active set, that is  $\dof(\hat\mu)(y) = \EE[\abs{A}]$, where $A$ is the smallest set for which $A = \mathrm{Supp}(\hat \beta)$ and $\hat \beta$ is a \lasso~solution.
In the case where the solution $\hat\beta$ is unique (for instance when $X$ is injective), then this simplifies to $\dive(\hat \mu)(y) = \|\hat\beta\|_0 = \abs{\enscond{i}{\hat\beta_i \neq 0}}$.
 
These results have been extended to more general classes of estimators based on various notions of low-complexity (e.g. sparsity) priors, such as for analysis-type \lasso~\cite{tibshirani2012degrees,vaiter-local-behavior}, projection on polytopes~\cite{MeyerWoodroofe} and more general convex sets~\cite{kato2009degrees}, singular value thresholding~\cite{candes2013unbiased} and penalized regression using quite general partly smooth regularizers (such as the group \lasso~and its variants)~\cite{vaiter2017degrees}.

Note that computation of the DOF for variational estimators such as the \lasso~and its variant is closely related to the theory of sensitivity analysis of optimization problems~\cite{BonnansShapiro2000}. Note however that the setting of the \blasso~problem we consider next is more intricate, since it corresponds to the study of an infinite dimensional optimization problem over a non-reflexive Banach space (the space of Radon measures), where usual tools from differential calculus are not available.

\subsection{The curse of discretization}
\label{sec-curse}

In many recent methods developed in machine learning and imaging sciences, $X$ is a finite dimensional approximation of some continuous operator, and one could attempt to obtain increased accuracy by using an increasingly finer grid (letting $p \to \infty$).
This setting usually corresponds to ``over-parameterized models'' where the $p$ columns of $X = \Phi_\Xx$ are samples $\phi(x_i) \in \RR^n$ for some locations $\Xx \eqdef (x_i)_{i=1}^p$ in a parameter space $x_i \in \Om$. Here $\phi : \Om \rightarrow \RR^n$ is a continuous function specifying the parameterization of the linear model. 
Two typical examples of such a settings are:
\begin{itemize}
	\item \textbf{Super-resolution in imaging:} in this case, $\Xx=(x_i)_{i=1}^p$ is a grid on which one aims at recovering Dirac masses composing the signal or the image $\be_0$ to recover from the noisy measurements $y = \Phi_\Xx \be_0 + w$ ($w \in \RR^n$ being a random vector accounting for measurement noise).
		A first popular example on $\Om=\RR^d$ ($d=1$ for signals and $d=2$ for images) are (sampled) convolutions, where $\phi(x) = ( \psi(x-z_j) )_{j=1}^n$ (where $z_j \in \RR^d$ are measurement locations, for instance pixels for an image, and $\psi$ is the point-spread function).
		A second, closely related, example is the one of Fourier measurements on a periodic domain $\Xx=(\RR/\ZZ)^d$, where $\phi(x) = ( e^{  2\imath\pi \dotp{x}{k_j}} )_{j}$ , where $k_j \in \ZZ^d$ are the measured frequencies.
		Since for simplicity, we only consider real-valued measurements in this article, assuming symmetric frequencies $k_{-j}=k_j$ and $k_0=0$, this example can be equivalently written using
		$\phi(x) = \pa{1, \pa{\sqrt{2}\sin(2\pi \dotp{x}{k_j})}_{j=1}^{f}, \pa{\sqrt{2}\cos(2\pi \dotp{x}{k_j})}_{j=1}^{f} }$, which defines a set of $n=2f+1$ real measurements.
		
	\item \textbf{Multi-layer perceptron with a single hidden layer:} given $n$ pairs of features/values $(a_j,y_j) \in \RR^{d} \times \RR$, the goal is to train a network $f_{\be,\Xx}(a) = \sum_{i=1}^p \be_i \xi(\dotp{a}{x_i})$ so that $f_{\be,\Xx}(a_j) \approx y_j$.
		Here $\Xx=(x_i)_{i=1}^p \subset \RR^d$ are the $p$ neurons composing the first (hidden) layer, while $(\be_i)_{i=1}^p$ are the scalar weights compositing the second layer. The function $\xi : \RR \rightarrow \RR$ is a pointwise non-linearity, the most popular one being the ReLu $\xi(r)=\max(r,0)$. These are the parameters to be trained from the data, and this corresponds to using $\phi(x) = ( \xi(\dotp{a_i}{x}) )_{i=1}^n$.
		When the neurons $(x_i)_i$ are fixed, sparsity-regularized empirical risk minimization over the second layer weights $\be$ then corresponds to solving a \lasso~problem~\eqref{eq-lasso}.
		Training the first layer neurons $(x_i)_i$ is non-convex, and this is why it makes sense, as we explain next and following~\cite{bach2017breaking}, to rather consider a convex problem on the space of measures.
\end{itemize}

As the discretization $p$ of the model increases, the design matrix $X=\Phi_\Xx$ becomes increasingly coherent (the columns $\phi(x_i)$ being highly correlated), so that both the theoretical analysis and the discrete numerical optimization solvers for the \lasso~\eqref{eq-lasso} become inefficient. 
An typical example of these difficulties is that, even for well specified deterministic observations  $y = X \beta_0$ (generated with a sparse vector $\|\be_0\|_0=k$), the solutions $\hat\be$ of the \lasso~\eqref{eq-lasso} are in practice composed of much more than $k$ Diracs when $p$ is large (and the presence of noise further aggravates this problem). This is made precise in~\cite{duval2017sparse}, where it is shown that for a 1-D deconvolution problem, $\|\hat\be\|_0=2k$, so that the degrees of freedom is the double of the number of spikes. This however should come as no surprise, since the ``intuitive'' number of degrees of freedom should take into account both positions and amplitudes, and is thus expected to be much larger than $k$. 
These important observations thus raise the question of properly and stably defining a consistent notion of degrees of freedom for these over-parameterized models. It is the purpose of this article to do so, and we show that it can be achieved in a painless way by considering a continuous generalization of the \lasso.

\subsection{Off-the-grid approaches and \blasso}
 
In order to alleviate these issues, a recent trend is to rather consider an ``off-the-grid'' approach, where one does not discretize the operator, but instead optimize over a sparse set of positions $(x_i)_i$ and amplitudes $(\be_i)_i$. In order to maintain the convexity of the \lasso~problem (which is non-convex with respect to the position), one rather optimizes over the space of Radon measures. One thus aims at recovering a sparse discrete measure of the form $\hat\mes = \sum_{j=1}^k \beta_j \delta_{x_j}$, by solving the following optimisation problem
\begin{equation}\label{eq:blasso}
	\min_{\mes \in \Mm(\Om)} \frac{1}{2}\norm{\Phi \mes - y}^2_2 + \la \abs{\mes}_{TV}. \tag{$\Pp_\la(y)$}
\end{equation}
Here,  $\Mm(\Om)$ denotes the whole space of Radon measures (in particular not only sparse ones) on a parameter domain $\Om$ (assumed for simplicity to be a subset of $\RR^d$).  The total variation norm of a  measure $\mes\in\Mm(\Om)$ is defined by 
$$
	\abs{\mes}_{TV} \eqdef \sup\enscond{\int_\Om f(x)\mathrm{d}\mes(x)}{\norm{f}_\infty\leq 1, \; f\in C_0^\infty(\Om)},
$$
and is equal to the total mass of the absolute value $\abs{\mes}_{TV} = |\mes|(\Om)$.
It generalizes the discrete $\ell^1$ norm in the sense that $\abs{\sum_{j=1}^k \beta_j \delta_{x_j}}_{TV} = \norm{\be}_1$.
The linear operator $\Phi:\Mm(\Om) \to \RR^n$ is defined by 
$\Phi \mes = \int_\Om \phi(x) \mathrm{d} \mes(x)$ for some continuous function $\phi \in \Cder{}(\Om; \RR^n)$. 
This formulation is the so-called Beurling \lasso~(\blasso)~\cite{deCastro-exact2012}, also proposed in \cite{bredies-inverse2013}. The properties of this estimator have been extensively studied in  \cite{candes-superresolution2013,candes-towards2013,azais-spike2014,duval2015exact}.
This is an extension of the \lasso, since the \lasso~can be obtained by restricting the optimisation over the measures supported on a discrete and finite grid $\Xx = (x_j)_{j=1}^p$ and letting $X$ be the matrix associated with the finite dimensional operator:  
\begin{equation}\label{eq:Phix}
	\Phi_\Xx: \beta\in  \RR^p \mapsto \sum_j  \beta_j \phi(x_j) 
	= \int \phi(x) \mathrm{d}\pa{\sum_{j=1}^p \beta_j \delta_{x_j}}
\end{equation}

In contrast to the \lasso~(as mentioned  above), it is known that under certain conditions, the \blasso~allows for the recovery of exactly $k$ components. 
If $y=\Phi( \sum_{j=1}^k \beta_j \delta_{x_j} ) + w$ and the additive noise $w$ is small enough, under some non-degeneracy condition, it is indeed shown in~\cite{duval2015exact} that the solution $\hat\mes$ of~\eqref{eq:blasso} is unique and is a discrete measure composed of $k$ diracs. 
This important stability property makes the \blasso~a perfect fit to define a meaningful notion of degrees of freedom for over-parameterized models.

The goal of this paper is thus to study the degrees of freedom associated to the estimator  $\hat \mu(y) \eqdef \Phi \hat  \mes$ where $\hat \mes$ is a solution  to \eqref{eq:blasso}. Again,  
$\hat \mu(y)$ is unique due to strong convexity of $\norm{\cdot}_2^2$, even if $\hat \mes$ is not unique. 
One surprising outcome of our analysis is that although the number of recovered parameters is still $2k$ (when $d=1$), since there are $k$ unknown positions $x$ and $k$ unknown amplitudes $\beta$, the degrees of freedom can be shown to be strictly smaller than $2k$.

%

%

\subsection{Contributions}

Our first theoretical result is Theorem~\ref{prop:div_formula}, which states a formula for $\dive(\hat \mu)$ under the assumption that one has access to a family of solutions of the \blasso~which is a smooth function of the observations $y$.
Our second and main contribution is Theorem~\ref{thm:cont_ext_supp1}, which shows that this assumption is valid (and thus the formula can be used) outside a set $y \notin \Hh$ of degenerate observations. 
Our last result is Theorem~\ref{thm:main2} which presents a formula for the degrees of freedom of the \blasso~in the case of 1-D Fourier measurements in terms of the number of recovered parameters and the curvature of the dual solution. 

En route to proving this result, we derive some results on the smallest active support which are of independent interest:
\begin{itemize}
\item With a nondegeneracy condition (Assumption (A1) in Section \ref{sec:divergence}) in the general setting, we prove that almost everywhere, the smallest active support evolves along a smooth path.
\item Without the nondegeneracy condition in the case Fourier sampling in dimension $d=1$, we prove that  almost everywhere, the smallest active support evolves along a smooth path.
\item We present numerical examples to verify our theoretical results. For the cases of sampling Fourier coefficients and the training of a 2-layers neural network, we demonstrate that our proposed divergence formula provides a tight estimate of  the risk, and that the dof is in general much smaller than the number of recovered parameters.
\end{itemize}

\subsection{Outline}

In Section \ref{sec:formal}, we assume that the recovered amplitudes and positions move along a differentiable  path and compute the divergence. The rest of this paper is dedicated to establishing cases where this smoothness assumption is valid. In Section \ref{sec:divergence}, we show that under a nondegeneracy condition at $y$, the solution path is indeed locally smooth. In Section \ref{sec:fourier}, we restrict to the case of sampling Fourier coefficients in dimension $d=1$, and show that the solution path is smooth around almost every $y \in \RR^n$ and hence, the divergence formula presented in Section \ref{sec:formal} is indeed valid and this allows for a closed form expression for the degrees of freedom.


\newcommand{\sgn}{s}

\section{The \blasso}

In this section, we recall some properties of the \blasso~and introduce some notations which are used throughout this article. We refer to \cite{bredies-inverse2013,deCastro-exact2012,duval2015exact} for more details on theoretical properties of the \blasso.

\subsection{Dual problem}

Let us first show that $y\mapsto \hat \mu(y)$ is a Lipschitz function and is thus differentiable almost everywhere. This is a direct consequence of the dual formulation  to  \eqref{eq:blasso}:
\begin{equation}\label{eq:dual}
\min_{\norm{\Phi^* p}_\infty \leq 1} \norm{p - y/\la}_2 \tag{$\Dd_\la(y)$}
\end{equation}
is the projection of $y/\lambda$ onto a closed convex set. Note that \eqref{eq:dual} has a unique solution $p_y$, and moreover, the dual and primal solutions are related by 
\begin{equation}\label{eq:primal-dual}
p_y\eqdef \frac{y -\Phi \mes_y }{\la} \in \argmin \Dd_\la(y)\qandq \mes_y \in \argmin \Pp_\la(y)
\end{equation}
In particular, we can write for any primal solution $\mes_y$, 
$$\Phi \mes_y = (\Id - P_C)(y/\lambda)$$ 
where $P_C$ is the projection onto the convex set $\enscond{p}{\norm{\Phi^* p}_\infty \leq 1}$.  So,
$$
\norm{\Phi (\mes_y - \mes_{y'}) } = \frac{1}{\lambda}\norm{(\Id - P_C)(y-y')} \leq  \frac{1}{\la}\norm{y-y'}.
$$ 
and hence, $y\mapsto \hat \mu(y)$ is a Lipschitz function and is differentiable almost everywhere. 
However, to derive an explicit formula for the degrees of freedom, we need to prove that  the recovered amplitudes $\be=(\be_j)_j$ and positions $\Xx=(x_j)_j$  are Lipschitz (as functions of $y$). 
Note that given $n$ measurements, there always exists a primal solution which is a discrete measure made of at most $n$ Diracs~\cite{fisher1975spline,unser2017splines,boyer2019representer}.

\subsection{Dual certificates and extended support}

Given $y$, its dual certificate is 
\eql{\label{eq-defn-dual-certif}
	\eta_y \eqdef \Phi^* p_y = \Phi^* \pa{\frac{y - \Phi \mes_y}{\la}}, 
}  
where $p_y$ is the dual solution described in \eqref{eq:primal-dual} and $\mes_y$ is any primal solution. Since $p_y$ is unique, $\eta_y$ is unique even if $\mes_y$ is not. Moreover, 
$$
	\eta_y \in \partial \abs{\mes_y}_{TV}
$$
where $ \partial \abs{\mes}_{TV}$ denotes the subdifferential of $\abs{\cdot}_{TV}$ at $\mes$. It can be shown (see for instance~\cite{duval2015exact}) that 
$$
\partial \abs{\mes}_{TV}  = \enscond{f\in \Cc(\Omega)}{\int f(x) \mathrm{d}\mes(x) = \abs{\mes}_{TV}\qandq \norm{f}_\infty \leq 1}.
$$
The extended support at $y$ is defined to be
$$
	\Ee_y \eqdef \enscond{x}{\abs{\eta_y(x)}=1 }.
$$
Given any solution $\mes_y$ of \eqref{eq:blasso}, we have that $\mathrm{Supp}(\mes_y) \subseteq \Ee_y$ and $\int \eta_y(x) \mathrm{d}\mes_y(x) = \abs{\mes_y}_{TV}$. In particular, if $\mes_y = \sum \beta_j \delta_{x_j}$, then $\eta_y(x_j) = \sign(\beta_j)$.

\subsection{Notations}

Given $f:\RR^n\to \RR$, its gradient is written as $\nabla f(y) = (\partial_i f(y))_{i=1}^n\in \RR^n$; and given a differentiable vector-valued function $f: \RR^n\to \RR^m$, $f(y) = (f_i(y))_{i=1}^m$, its Jacobian is  the $m\times n$ matrix whose rows are $\nabla f_j$: $$
J_f(y) = \begin{pmatrix}
\nabla f_1(y) & \nabla f_2(y) &
\cdots &
\nabla f_m(y)
\end{pmatrix}^\top \in \RR^{m\times n}
$$

For $n\in\NN$, let $[n]\eqdef \ens{1,2,\ldots, n }$. We write $M\succeq 0$ to denote that a matrix $M$ is positive semi-definite and $M\succ 0$ to denote a matrix as positive definite. Given two positive semi-definite matrices $M$ and $N$, write $M\succeq N$ (resp. $M\succ N$) to mean $M-N\succeq 0$  (resp. $M-N\succ 0$).
Given $M\in\RR^{n\times m}$, $I\subseteq [n]$, $J\subseteq[m]$, let $M_J$ denote the matrix with columns restricted to the index set $J$ and $M_{(I,\cdot)}$ denote the matrix with rows restricted to the index set $I$.

Recall for $\Xx=(x_i)_i \in \Om^k$ the definition of $\Phi_\Xx$ in \eqref{eq:Phix}. We also 
define the derivative mapping $\Phi_\Xx^{(1)}: \RR^{kd}\to \RR^n$, so that given $(\beta_j)_{j=1}^k$ where $\beta_j\in \RR^d$,
\begin{equation}
	\Phi_\Xx^{(1)} \beta \eqdef \pa{\sum_{j=1}^k \dotp{\beta_j}{\nabla \phi(x_j)}}_{k=1}^n.
\end{equation}
We define for $\Xx = (x_i)_i \in \Om^k$,
$\Gamma_\Xx \eqdef [\Phi_{\Xx}, \Phi_{\Xx}^{(1)}] \in \RR^{n\times k(d+1)}$. Given $\beta \in \RR^k$ and $\Xx \in \Om^k$, we denote $\mes_{\beta,\Xx}\eqdef \sum_{j=1}^k \beta_j \delta_{x_j}$.

\section{Formal computation of the degrees of freedom}\label{sec:formal}

We first derive an expression for the divergence 
under the assumption that there exist solutions such that the number of recovered spikes $k$ is constant locally, and the recovered positions and amplitudes are differentiable.

\begin{thm}\label{prop:div_formula}
Let $\bar y\in \RR^n$ and suppose that there is a neighbourhood $U$  of $\bar y$ such that for all $y'\in U$, there exists $k\in \NN$, $\beta_{y'} \in \RR^k$ and $\Xx_{y'} \in \Om^k$ such that $\mes_{\beta_{y'}, \Xx_{y'}}$ solves $(\Pp(y'))$. We also assume that the mapping $y \in U \mapsto (\beta_y, \Xx_y) \in \RR^k \times \Om^k$ is differentiable. 
Writing  $\bar \Xx = (\bar x_j)_{j=1}^k= \Xx_{\bar y}$ and $\bar \beta = \beta_{\bar y}$,   $Q = (Q_j)_{j=1}^k$ with $Q_j = -\frac{\lambda}{\bar \beta_j}\nabla^2 \eta_{\bar y}(\bar x_j)\in \RR^{d\times d}$ (where $\eta_{\bar y}$ is defined in~\eqref{eq-defn-dual-certif}), assume that
$
M \eqdef \Gamma_{\bar \Xx}^* \Gamma_{\bar \Xx}   + 
\begin{pmatrix}
	0_{k\times k} & 0_{k\times kd} \\
	0_{k d\times k} & \diag(Q)
\end{pmatrix}$ 
is invertible, where $\diag(Q)$ is the block-diagonal matrix with $(Q_j)_j$ along the diagonal.
Then,
\eql{\label{eq-formula}
	\dive(\hat \mu)(\bar y)  = \tr\pa{\Gamma_{\bar \Xx} M^{-1} \Gamma_{\bar \Xx}^*}.
}
\end{thm}

Before proving this first theorem, let us mention an important consequence, that the empirical degrees of freedom is smaller, and in general strictly smaller, than the natural number of parameters $P \eqdef (d+1)k$ of a $k$-sparse model in dimension $d$.
Indeed, in practice (see Section \ref{sec:numerics}), we find that $\dive(\hat \mu)(\bar y)$ can be much smaller than $P$. This can intuitively been seen from formula~\eqref{eq-formula}, since the value of  $\dive(\hat \mu)(\bar y)$ is mostly driven  by the rank of $\Gamma_{\bar \Xx}$ and the curvature of the dual certificate $\eta_{\bar y}$ at the recovered support points. 

\begin{cor}
One has 
\eql{ \label{eq-bound-dof}
	0\leq \dive(\hat \mu)(\bar y) \leq  \rank(\Gamma_{\bar \Xx}) \leq P \eqdef (d+1)k.
}
If furthermore $\Gamma_{\bar \Xx}$ is injective and $\nabla^2 \eta_{\bar y}(\bar x_j)$ is invertible for all $j\in[k]$, then 
\eq{
	0\leq \dive(\hat \mu)(\bar y) <  \rank(\Gamma_{\bar \Xx}) = P.
}
\end{cor} 

\begin{proof}
Since $\Gamma_{\bar \Xx} M^{-1} \Gamma_{\bar \Xx}^*$ is positive semi-definite, $\dive(\hat \mu)(\bar y)\geq 0$.
A quick computation  (detailed in Appendix \ref{sec:comment-trace}) shows that
$$
\dive(\hat \mu)(\bar y)  = \rank(\Gamma_{\bar \Xx}) - \tr\pa{ \Pi_{\Im(\Gamma_{\bar \Xx})} \begin{pmatrix}
0 & 0\\
0 & \diag(Q)
\end{pmatrix} M^{-1}    } \leq \rank(\Gamma_{\bar \Xx}).
$$
Suppose that $\Gamma_{\bar \Xx}$ is injective, then the above expression reduces to 
$$
\dive(\hat \mu)(\bar y) = (d+1)k - \tr\pa{ \begin{pmatrix}
0 & 0\\
0 & \diag(Q) 
\end{pmatrix} M^{-1}    }
$$
Assuming that $\frac{-1}{\beta_j}\nabla^2 \eta_{\bar y}(\bar x_j)$ is positive definite for all $j\in[k]$, it follows that $Q$ is positive definite and hence, $\dive(\hat \mu)(\bar y) < (d+1)k$: Indeed, we can write $M^{-1} = U \diag(\sigma) U^*$ where $U$ is unitary and $\sigma \in \RR_+^P$. Writing $U = \begin{pmatrix}
U_1 & U_2\\
U_2^* & U_3
\end{pmatrix}$ and $\Sigma_1 = \diag(\sigma_j)_{j=1}^k$ and $\Sigma_2 = \diag(\sigma_j)_{j=k+1}^{P}$, we have $\tr\pa{ \begin{pmatrix}
0 & 0\\
0 & \diag(Q) 
\end{pmatrix} M^{-1}} = \tr\pa{Q^{\frac12} (U_2^*  \Sigma_1 U_2 + U_3\Sigma_2U_3^* )Q^{\frac12}}>0$.
\end{proof}

\begin{proof}\textit{(of Theorem~\ref{prop:div_formula})}
We write $\Xx_y \eqdef (x_i(y))_{i=1}^k$, where $x_i(y)\in \RR^d$.  In the following, to simplify the notation, we sometimes drop to subscript $y$ to write $\beta = \beta_y$, $\Xx = \Xx_y$. 
Recall that we denote by $J_\beta$ the Jacobian of $\beta$ and $J_{x_i}\in \RR^{k\times n}$ the Jacobian of $x_i$. We thus write
\begin{align*}
J_\Xx \eqdef \begin{pmatrix}
J_{x_1}\\
\vdots\\
J_{x_k}
\end{pmatrix} \in \RR^{kd\times n} 
\end{align*}

Recall that  $\eta_y \eqdef -\Phi^*(\Phi_\Xx \be - y)/\lambda$.
Because of the interpolation condition $(\eta_y(x_i))_i=s_y\eqdef \sign(\beta_y)$ and since $|\eta_y| \leq 1$ so that $\nabla \eta_y(x_i)=0$, one has \begin{equation}\label{eq:opt-blasso}
\Phi_\Xx^* \Phi_\Xx \beta = \Phi_\Xx^* y - \lambda s_y \qandq (\Phi_\Xx^{(1)})^* \Phi_\Xx \beta = (\Phi_\Xx^{(1)})^* y.
\end{equation}
 Note that by continuity  of $\beta$, locally, $s_y$ is constant.

Let $D_\beta \eqdef \diag(\beta)\otimes \Id_{d\times d}$.
Differentiating \eqref{eq:opt-blasso} with respect to $y$, we obtain
\begin{equation}\label{eq:deriv1}
  \Phi_\Xx^* \Phi_\Xx J_\beta +\Phi_\Xx^* \Phi_\Xx^{(1)} D_\beta J_\Xx = \Phi_\Xx^*.
\end{equation}
and
\begin{equation}\label{eq:deriv2}
\diag(Z)J_\Xx + (\Phi_\Xx^{(1)})^*\Phi_\Xx J_\beta +(\Phi_\Xx^{(1)})^* \Phi_\Xx^{(1)} D_\beta J_\Xx = (\Phi_\Xx^{(1)})^*,
\end{equation}
where we define $Z_i \eqdef - \lambda \nabla^2 \eta_y(x_i)\in \RR^{d\times d}$ and $\diag(Z)\in \RR^{kd\times sd}$ is the block diagonal matrix whose $i^{th}$ diagonal block is $Z_i$. Note that we have used the fact that $\nabla \eta_y(x_i) = 0$ in deriving \eqref{eq:deriv1}.

Writing $\Gamma_\Xx \eqdef [\Phi_\Xx, \Phi_\Xx^{(1)}] \in \RR^{k(d+1)\times n}$ and
\begin{equation}\label{eq:M}
M \eqdef \Gamma_\Xx^* \Gamma_\Xx   + \begin{pmatrix}
0_{s\times s} & 0_{s\times sd} \\
0_{sd\times s} & \diag(Z) D_\beta^{-1}
\end{pmatrix},
\end{equation} equations \eqref{eq:deriv1} and \eqref{eq:deriv2} can be written as
\begin{equation}\label{eq:deriv}
\Gamma_\Xx^* \Gamma_\Xx \binom{J_\beta}{D_\beta J_\Xx} +\binom{0}{\diag(Z) J_\Xx} = \Gamma_\Xx^*
\iff M \binom{J_\beta}{D_\beta J_\Xx} =  \Gamma_\Xx^*.
\end{equation}
From $Z_j = -\lambda \nabla^2 \eta_y(x_j)$ and $\sign(\beta_j)\eta_y(x_j) = 1$, we have $\sign(\beta_j) \nabla^2 \eta_y(x_j) \preceq 0$ and $Z_j/\beta_j \succeq 0$ for all $j$.  So, $M$ is positive semi-definite.  

Suppose now that $M$ is also invertible. Then, 
$$
	\nabla_y \hat \mu(y) = \Phi_\Xx^{(1)} D_\beta J_\Xx + \Phi_\Xx J_\beta =  
	\Gamma_\Xx  \binom{J_\beta}{D_\beta J_\Xx} = \Gamma_\Xx M^{-1} \Gamma_\Xx^*
$$ 
and the divergence of $\hat\mu$ is
\begin{equation}\label{eq:dof}
\dive(\hat \mu)(y) =  \sum_i \frac{\mathrm{d}\hat\mu_i}{\mathrm{d}y_i}(y) = \tr\pa{\Gamma_\Xx M^{-1} \Gamma_\Xx^*} \geq 0.
\end{equation}
Note also that from the left-hand-side equation of \eqref{eq:deriv},  we can write
\begin{align*}
\Gamma_\Xx \binom{J_\beta}{D_\beta J_\Xx} &= \Gamma_\Xx \Gamma_\Xx^\dagger - \Gamma_\Xx(\Gamma_\Xx^* \Gamma_\Xx)^\dagger \binom{0}{\diag(Z) J_\Xx}\\
&=  \Gamma_\Xx \Gamma_\Xx^\dagger - \Gamma_\Xx(\Gamma_\Xx^* \Gamma_\Xx)^\dagger 
\begin{pmatrix}
0 & 0\\
0 & \diag(Z) D_\beta^{-1}
\end{pmatrix} \binom{J_\beta}{D_\beta J_\Xx}\\
&= \Gamma_\Xx \Gamma_\Xx^\dagger - \Gamma_\Xx(\Gamma_\Xx^* \Gamma_\Xx)^\dagger 
\begin{pmatrix}
0 & 0\\
0 & \diag(Z) D_\beta^{-1}
\end{pmatrix} M^{-1} \Gamma_\Xx^*.
\end{align*}


\end{proof}

In order for equation \eqref{eq-formula} to be valid, we need to prove that the matrix $M$  defined in~\eqref{eq:M} is invertible and that $y\mapsto \beta_y$ and $y\mapsto \Xx_y$ are Lipschitz. This is the subject of the subsequent sections.


\section{Divergence of the  \blasso}\label{sec:divergence}

In this section, we show that the divergence of the \blasso~can be explicitly computed in the case where  $y\in \RR^n$ is such that:
\begin{itemize}
\item[(A1)] The extended support $\Ee_{y}$ is a discrete set consisting of $m$ points for some $m\in \NN$ and $\eta_y(x) \nabla^2\eta_y(x) \prec 0$ for all $x\in \Ee_y$.
\end{itemize}
It is known \cite{duval2015exact} that this  ensures that the size of the extended support  remains constant locally around $y$. In particular, 
there exists a neighbourhood $U$ around $y$ such that for all $y'\in U$, $\Ee_{y'}$ is also discrete with $m$ points and $y' \in U \mapsto \Ee_{y'}$ is a continuous mapping.

If one additionally has that $\Phi_\Ee$ is injective, then  uniqueness and continuity of the recovered positions and amplitudes is guaranteed. Establishing support stability is less clear in the case where injectivity of $\Phi$ on $\Ee$ fails. Nonetheless,  in this section, we show that one can still obtain a support stability result on a subset of the extended support, provided that $y$ satisfies (A1) and does not lie in the following set $\Hh$: 
\begin{equation}\label{eq:setH-general}
\begin{split}
\Hh \eqdef \bigcup_{k=1}^\infty \bigcup_{\sigma\in \{-1,+1\}^k} \bigcup_{S\subseteq [k]} \bigcup_{I \subseteq S}  \mathrm{Bd}(\Pi_Y (\Qq_{k,S,I,\sigma}))
\end{split}
\end{equation}
where $\mathrm{Bd}(S)$ is the boundary of a set $S$, $\Pi_Y:  \RR^n \times \RR^k \times \Om^k \to \in \RR^n$ is the projection mapping $(y,a,\Ee) \mapsto y$ and
\begin{align*}
\Qq_{k,S,\sigma,I}\eqdef \Bigg\{(y,a,\Ee)&\in \RR^n\times \RR^k \times \Om^k;\;\\
&\Phi^*(y - \Phi_\Ee a)\in\lambda \partial \abs{\mes_{a,\Ee}}_{TV} , M\eqdef \Phi_\Ee, \rank(M) = \abs{S}\\
& \pa{ (M^\dagger  + M_{S}^{\dagger} M_{S^c}  (M)^\dagger_{(S^c,\cdot)})(y - (M^*)^\dagger \lambda \sigma )   }_I = 0_I  
\Bigg\}.
\end{align*}

\begin{rem}
In the finite dimensional case of the LASSO, a divergence formula is established in~\cite{dossal2013degrees} outside a set of measure zero, which was shown to be a union of hyperplanes. In our case, 
we wish to show that $\Hh$ (which is no longer composed of affine spaces) is a set of Lebesgue measure zero in $\RR^n$. Intuitively, this should follow from the fact that $\Hh$ is a countable union of boundaries of subsets of $\RR^n$, and this is of zero measure if these boundaries do not ``oscillate'' too wildly. 
%
This is ensured for quite general class of models $\phi(x)$ if they are semi-algebraic sets (which is the case for Fourier measurements and neural networks with a ReLu non-linearity), and more generally (for instance for Gaussian functions), using the notion of definable sets in o-minimal geometry \cite{coste1999introduction}, a generalization of real algebraic geometry.
The construction of this set $\Hh$ is inspired by the construction of the so-called transition space in \cite{vaiter2017degrees}. We however highlight that arguments in \cite{vaiter2017degrees} are valid only in the finite-dimensional setting since in particular they rely on enumerating all possible active manifolds, which is not possible in our setting.  
In Appendix \ref{sec:ominimal}, we recall some notions from o-minimal geometry and show that $\Hh$ is of zero measure under the assumption that $x \mapsto \phi(x)$ is \textit{definable}.
\end{rem}

\begin{rem}
Intuitively, in order to establish smoothness of the recovered parameters $\beta_y$ and $\Xx_y$, we need to require that locally around $y$, there exists solutions such that the number of recovered parameters remain constant, and the rank of $\Phi$ restricted to the extended support has constant rank. This is the idea behind the definition of the sets $\Pi_Y(\Qq_{k,S,I,\sigma})$, so we have differentiability of the recovered parameters away from the boundaries of such sets.
\end{rem}

We first show that one can construct a solution which is supported on a subset $\Aa$ of the extended support such that $\Phi_\Aa$ is injective, a similar statement is proved in \cite[Appendix B]{rosset2004boosting}, however, we include a proof for completeness.

\begin{lem}\label{lem:inj}
Suppose that $\Ee_y$ is discrete. Then, there exists $\Aa\subset \Ee_y$ and a solution to \eqref{eq:blasso} with support $\Aa$  such that that $\Phi_\Aa$ is injective.
\end{lem}

\begin{proof}
Since any solution to \eqref{eq:blasso} has $\mathrm{Supp}(\mes) \subseteq \Ee_y$, there exists a solution of the form $\mes =  \mes_{\beta,\Aa}$, where $\Aa\subset \Ee_y$. Suppose that $\Phi_\Aa$ is not injective. Then there exists $b$ such that $\Phi_\Aa b = 0$, and so, for any $t\in \RR$, by defining $\mes_t \eqdef \mes_{\beta+tb,\Aa}$, we have $\Phi \mes_t = \Phi \mes$.     Moreover, for all $t$ sufficiently small, we have $s_j \eqdef \sign(\beta_j) = \sign(\beta_j+tb_j)$ for all $t$ sufficiently small, and since $\mes$ is a solution,  $\norm{\beta}_1 \leq  \norm{\beta+tb}_1 = \sum_j (\beta_j+tb_j) s_j = \norm{\beta}_1 + t \sum_j b_j s_j$ for all  $t$ sufficiently small which implies that $\sum_j b_j s_j = 0$.  
For $v\in\{+,-\}$, let $A^{+,v}\eqdef \enscond{j}{b_j \sign(\beta_j) >0,  \sign(\beta_j) = v}$ and  $A^{-,v}\eqdef \enscond{j}{b_j \sign(\beta_j) \leq 0,  \sign(\beta_j) = v}$. Note that either  $A^+ \eqdef A^{+,+}\cup A^{+,-}  \neq \emptyset$ or  $A^- \eqdef A^{-,+}\cup A^{-,-} \neq \emptyset$ .
 
 Suppose that either $A^-=\emptyset$ or $\norm{b_{A^-}}_\infty = 0$, then for all $j$, either $b_j \sign(\beta_j) >0$ or $b_j = 0$, so $\sign(\beta_j + tb_j) = \sign(\beta_j)$ for all $t>0$ and all $j$. Suppose that $A^-\neq \emptyset$ and $\norm{b_{A^-}}_\infty \neq 0$, and let $t \eqdef \min_{i\in A^-} \abs{\beta_i}/ \norm{b_{A^-}}_\infty>0$. Then, clearly, $\sign(\beta_j + t b_j ) = \sign(\beta_j)$ for all $j\in A^+$. For $j\in A^{-,+}$,  $\beta_j >0$ and $\beta_j + t b_j \geq \beta_j - \min_i \abs{\beta_i} \geq 0$, and for $j\in A^{-,-}$, $\beta_j <0$ and $\beta_j + t b_j \leq \beta_j + \min_i \abs{\beta_i} \leq 0$. In particular, there exists $t>0$ such that for all $j$, either $\sign(\beta_j+ t b_j) = \sign(\beta_j)$ or $\beta_j + tb_j = 0$. Let $t_1>0$ be the largest such $t$. If $\abs{\beta+ t_1 b}$ has all nonzero entries, then we can repeat this argument on $\beta' \eqdef \beta + t_1 b$ to obtain $t_2>0$ such that $\sign(\beta'+ t_2 b) = \sign(\beta') = \sign(\beta)$. But this is a contradiction to $t_1$ being the largest such $t$. Therefore, $\mes_{\beta+t_1 b,\Aa}$ is supported on at least one less point than $\mes_{\beta,\Aa}$. 
\end{proof}

We now state and prove our first main theorem, which provides sufficient conditions under which Proposition \ref{prop:div_formula} can be applied.

\begin{thm}\label{thm:cont_ext_supp1}
Assume that (A1) holds and $y\not\in \Hh$.
Let $J\subseteq [m]$ of a set of cardinality $k$, such that $\Aa_y \eqdef (\Ee_{y})_J$ satisfies that $\Phi_{\Aa_y}$ is injective and $\mes_{\beta_y,\Aa_y}\in \argmin \Pp_\la(y)$ for some $\beta_y\in\RR^k$ having all non-zero entries (which is possible by Lemma \ref{lem:inj}).
 Then, there exists a neighbourhood $U$ of $y$ such that for all $y'\in U$, there exists $\beta_{y'}\in \RR^k$ and $\Aa_{y'}\in \Om^k$ such that $\mes_{\beta_{y'},\Aa_{y'}}$ solves $\Pp_\la(y')$. Moreover, the mapping $y'\mapsto (\beta_{y'}, \Aa_{y'})$ is $\Cc^1$. 
\end{thm}

\begin{rem}
Given $y\not \in \Hh$ such that (A1) holds,  the divergence formula \eqref{eq:dof} is valid with support $\Xx = \Aa_y$.
\end{rem}

The remainder of this section is devoted to the proof of this Theorem.
To prove it, we construct a $\Cc^1$ path for solutions to \eqref{eq:blasso} in a small neighbourhood of $y$.
 Let $s_y \eqdef (\eta_y)_{\restriction_{\Aa_y}}$.  Define the function
$$
F : (\beta,\Aa,y)\in \RR^{k}\times \Omega^k \times\RR^n \mapsto \Gamma_\Aa^* \pa{\Phi_\Aa \beta - y} + \lambda\binom{s_y}{0_{kd}}\in \RR^{k(d+1)}
$$
where $\Gamma_\Aa = [\Phi_\Aa,\Phi'_{\Aa}]$. 
We have $\partial_y F = -\Gamma_\Aa^*$, and writing  $u \eqdef (\beta,\Aa)$,
\begin{align*}
\partial_u F &= \Gamma_\Aa^* \Gamma_A \begin{pmatrix}
\Id_k & 0\\
0 & \diag(\beta)\otimes \Id_d
\end{pmatrix} + \begin{pmatrix}
\Id_k & 0\\
0 & \diag((z_i)_i)
\end{pmatrix} \\
&= \pa{\Gamma_\Aa^* \Gamma_\Aa  + \begin{pmatrix}
\Id_k & 0\\
0 & \diag((Z_i/\beta_i)_i)
\end{pmatrix}  } \begin{pmatrix}
\Id_{k} & 0\\
0 & \diag(\beta)\otimes \Id_{d}
\end{pmatrix}
\end{align*}
where $Z_i = \dotp{\Phi_\Aa \beta - y}{\nabla^2 \phi(x_i)}\in\RR^{d\times d}$ if $\Aa = (x_i)_{i=1}^k$. Therefore, $\partial_u F$ in invertible because
$$
\ker \pa{\Gamma_\Aa^* \Gamma_\Aa} \cap \ker\pa{\begin{pmatrix}
\Id_k & 0\\
0 & \diag((Z_i/\beta_i)_i)
\end{pmatrix}} = \{0\},
$$ since $\Phi^*_\Aa\Phi_\Aa$ is invertible, and $\frac{1}{\beta_i} Z_i \succ 0$ for all $i$. So,
we can apply the implicit function theorem to define a function $g$ in a small neighbourhood $U$ around $y$, such that $g: y'\in U\mapsto (\beta', \Aa')$ is a $\Cc^1$ function.

It remains to show that given $(\be', \Aa') = g(y')$, $\mes_{\be',\Aa'}$ is indeed a solution of $\Pp_\la(y')$.
To this end, given $y'\in U$, we simply need to construct a sparse solution made of $k$ diracs, with support $\Ss$ and amplitude $\alpha$ such that $\norm{\Ss - \Aa}$  and $\norm{\alpha-\be}$ are sufficiently small. Then by uniqueness of the implicit function $g$, this would allow us to conclude that $\alpha = \be'$ and $\Ss = \Aa'$.

We make use of the following lemma, which gives an explicit formula for a solution of \eqref{eq:blasso} when the extended support is discrete.

\begin{lem}\label{lem_sol}
Given $y$, let $\eta \eqdef \eta_y$ be its dual certificate,  $\Ee \eqdef \Ee_y$ be the extended support, and assume that $\Ee$ is a discrete point set. Let $s = \eta_{\restriction_\Ee}$.   
Any solution to \eqref{eq:blasso} can be written as $\mes_{\beta,\Ee}$ where
\begin{equation}\label{eq:lassosol}
\beta \eqdef (\Phi_\Ee)^\dagger (y - (\Phi_\Ee^*)^\dagger \lambda s) + b
\end{equation}
for some $b\in \ker(\Phi_\Ee)$.
Moreover, by defining
$$
\beta \eqdef (\Phi_\Ee)^\dagger (y - (\Phi_\Ee^*)^\dagger \lambda s),
$$
we have that
$\mes_{\beta,\Ee}$ is a solution to \eqref{eq:blasso}.
\end{lem}
\begin{proof}
Any solution of \eqref{eq:blasso} has support included in $\Ee$. Therefore, $\mes_{\beta, \Ee}$ is a solution of \eqref{eq:blasso} if and only if $\beta$ solves the following \lasso~problem:
\begin{equation}
\min_{\beta\in \RR^{\abs{\Ee}}} \norm{\beta}_1 + \frac{1}{2\la} \norm{\Phi_\Ee \beta - y}_2^2.
\end{equation}
Since $\Ee$ is also the extended support of this problem, we know from \cite{tibshirani2012degrees} that solutions are of the form \eqref{eq:lassosol} and that 
$\beta = (\Phi_\Ee)^\dagger (y - (\Phi_\Ee^*)^\dagger \lambda s)$ is a solution.
\end{proof}

\begin{proof}[Proof of Theorem \ref{thm:cont_ext_supp1}]
Let $\beta_y$ have support $J$ be such that $\mes_{\beta_y,\Ee_y}$ solves \eqref{eq:blasso}.
We first present some properties of   $\beta\eqdef \beta_y$:  Define $M^y \eqdef \Phi_{\Ee_y}$.
By Lemma \ref{lem_sol}, there exists  $b\in \ker(M^y)$ such that
$$
\beta = (M^y)^\dagger (y - ((M^y)^*)^\dagger \lambda s_y) + b, \qwhereq s_y = (\eta_y)_{\restriction \Ee_y}.
$$
Let $m \eqdef \abs{\Ee_y}$ and let $S\subseteq [m]$ be such that $S\supseteq J$, $\abs{S} = \rank(M^y)$ and $M^y_S$ is injective.
Since $\beta_{J^c} = 0$, we have $-b_{J^c} =  (M^y)^\dagger_{(J^c,\cdot)} (y - ((M^y)^*)^\dagger \lambda s_y)$. In particular, $-b_{S^c} =  (M^y)^\dagger_{(S^c,\cdot)} (y - ((M^y)^*)^\dagger \lambda s_y)$. Also, $$M^y_S b_{S} = - M^y_{S^c}b_{S^c} =M^y_{S^c} (M^y)^\dagger_{(S^c,\cdot)} (y - ((M^y)^*)^\dagger \lambda s_y)
$$
which implies that 
$$ b_{S} =  (M^y_{S})^{\dagger} M^y_{S^c} (M^y)^\dagger_{(S^c,\cdot)} (y - ((M^y)^*)^\dagger \lambda s_y)
$$
since $M_S^y$ is injective. Also, letting  $I \eqdef S\setminus J$, we have that 
$$
\beta_I = \pa{ \pa{(M^y)^\dagger +
 (M^y_{S})^{\dagger} M^y_{S^c} (M^y)^\dagger_{(S^c,\cdot)} }(y - ((M^y)^*)^\dagger \lambda s_y)}_I = 0_I.
$$
So, $y\in \Pi_Y(\Qq_{m,S,s_y,I})$.
Since $y\not\in \Hh$, it is in the interior of $\Pi_Y(\Qq_{m,S,s_y,I})$, we have that for some $\epsilon>0$ and all $y'$ in the ball $\Bb_\epsilon(y)$ of radius $\epsilon$,  there exists 
$\beta',\Ee'$ such that $\mes_{\beta',\Ee'}$ is a solution. So, $\Ee'$ is contained in the extended support $\Ee_{y'}$. But by continuity of the extended support (due to Proposition \ref{prop:cont_ext}), we must have that $\Ee'$ is exactly the extended support at $y'$. Define now $M^{y'}\eqdef \Phi_{\Ee'}$ and note that (again because $y$ is in the interior) $\rank(M^{y'})  = \rank(M^y)$ and it satisfies
\begin{equation}\label{eq:suppI0}
\pa{ \pa{(M^{y'})^\dagger +
(M^{y'}_{S})^{\dagger} M^{y'}_{S^c} (M^{y'})^\dagger_{(S^c,\cdot)}} f(y')
}_I = 0_I,
\end{equation}
where we define $f(y') \eqdef  (y' - ((M^{y'})^*)^\dagger \lambda s_y)$.

We now construct a solution for  $\Pp_\lambda(y')$ with support $\Aa' \eqdef \Ee'_J$.
By continuity of the extended support, $M^{y'} \to M^y$ and $M^{y'}_S$ is injective, and since rank is preserved,  $((M^{y'})^*)^\dagger \to ((M^{y})^*)^\dagger$.

Define $\bar b$ by
\begin{equation}\label{eq:b_dash0}
-\bar b_{S^c} = (M^{y'})^\dagger_{(S^c,\cdot)} f(y') \qandq  M^{y'}_S \bar b_{S} = -M^{y'}_{S^c} \bar b_{S^c}.
\end{equation}
 Note that the latter is a consistent definition because $\mathrm{Rank}(M^{y'}_S) = \abs{S} = \mathrm{Rank}(M^{y'})$ so $\mathrm{Col}(M^{y'}_{S^c}) \subset \mathrm{Col}(M^{y'}_S )$ (where $\mathrm{Col}(M)$ denotes the column space of matrix $M$). So, $\bar b \in \ker(M^{y'})$. Define $\bar \beta\in \RR^{m}$ by
$$
\bar \beta  \eqdef (M^{y'})^\dagger f(y') + \bar b
$$
where $f(y') = y' - ((M^{y'})^*)^\dagger \lambda s_y$. If we can show that
\begin{equation}\label{eq:toshow}
\forall j\in J, \; \sign(\bar \beta)_j = (s_y)_j \qandq \forall j\not\in J, \; \bar \beta_j = 0
\end{equation}
 then we must have that $\mes_{\bar \beta,\Ee'}$ has support $\Aa'$ and is a solution to $\Pp_\la(y')$. Moreover, we can then apply the implicit function theorem to conclude that the amplitudes and positions follow a $\Cc^1$ path locally around $y$.

To prove \eqref{eq:toshow}:
First, we have $\bar \beta_{J^c}= 0$ since \eqref{eq:suppI0} implies that $\bar \beta_I = 0_I$ and  \eqref{eq:b_dash0} implies that $\bar \beta_{S^c} = 0$. It remains to consider $\beta'_{J}$. We can write
$$
\bar \beta = (M^{y'})^\dagger f(y') +b + \bar b -b = \beta + (M^{y'})^\dagger f(y')  - (M^{y})^\dagger f(y) + \bar b -b .
$$
Note that $f$ and $\bar b$ are continuous as $y'$ changes, and since $\beta_J$ has all non-zero entries, $\sign(\bar \beta_J ) = \sign(\beta_J)$ when $y'$ is sufficiently close to $y$.

\end{proof}

\section{Degrees of freedom for Fourier sampling in 1D}\label{sec:fourier}

In this section, we consider the special case of sampling the Fourier coefficients up to some cut-off $f_c\in\NN$ of a 1-D real-valued measure (i.e. $d=1$), supported on $\Om=\TT \eqdef \RR/\ZZ$. This corresponds to using 
$\phi(x) = \pa{1, \pa{\sqrt{2}\sin(2\pi \dotp{x}{\ell})}_{\ell =1}^{f_c}, \pa{\sqrt{2}\cos(2\pi \dotp{x}{\ell})}_{\ell =1}^{f_c} }$ and $n= 2f_c+1$. Note that  we can write 
\begin{equation}\label{equiv_complex}
\Phi \mes = U \Psi \mes,\qwhereq	\Psi \mes \eqdef \pa{ \int_\TT e^{-2\imath\pi \ell x}\mathrm{d}\mes(x) }_{\ell =-f_c}^{f_c}
\end{equation}
and given $y_1,y_{-1}\in\CC^{f_c}$ and $y_0\in\CC$, $U$ is the unitary  mapping
$$
U: \begin{pmatrix}
y_1\\y_0\\y_{-1}
\end{pmatrix}\in \CC^{n} \mapsto \begin{pmatrix}
\frac{1}{\sqrt{2}}(y_1+y_{-1})\\ y_0 \\ \frac{1}{\mathrm{i}\sqrt{2}}(y_1-y_{-1})
\end{pmatrix} \in \CC^{n}.
$$
So, we can equivalently solve the \blasso~with $\Phi$ or with $\Psi$.


Note  that  elements in the image of $\Phi^*$ are  trigonometric polynomial of degree $f_c$, and can therefore have at most $f_c$ double roots, hence any discrete solution to \eqref{eq:blasso} is made of at most $f_c$ Diracs.

The key assumption in the previous section is that $\eta(x)\nabla^2 \eta_y(x)\prec 0$ for all $x\in \Ee_y$. This ensures continuity of the extended support and also ensures that $\Gamma_{\Ee_y}$ is full rank, and hence, allows for the use of the implicit function theorem in constructing a smooth path of solutions.    It is unclear that the set of $y$ for which this condition on the Hessian of $\eta_y$ fails is of measure zero.  However, in the case of  Fourier measurements in 1D, one can show (see Appendix \ref{app:fullrank}) that $\Gamma_{\Ee_y}$ is always of full rank and   the next proposition shows that the condition on the Hessian of $\eta_y$ can be relaxed.  If $\eta_y$ is not a constant function, then there exists some $\ell \in [n]$ such that the $\ell^{th}$ derivative of $\eta_y$ does not vanish.  Moreover, preservation of the vanishing derivatives as $y$ changes in a small neighbourhood is enough to guarantee continuity of the extended support. Given $y$, let $\eta_y$ be its dual certificate,  $\Ee_y$ be the extended support and $s^y \eqdef {\eta_y}_{\restriction_{\Ee_y}}$. Then, we have the following result.
\begin{prop}\label{prop:cont_ext}
Let $m\in\NN$.
Assume that  $y$ satisfies:
\begin{itemize}
\item[(i)] $\Ee_y = \{x_i\}_{i=1}^m$ is a discrete set,  and for each $i$, let $\ell_i\in \NN$ be such that $\forall \ell <\ell_i$, $\eta_y^{(2\ell)}(x_i)=0$ and $\eta_y^{(2 \ell_i)}(x_i)\neq 0$. 
\item[(ii)] there exists a neighbourhood $U$ around $y$ such that for all $y' \in U$,  there exists $m$ distinct points $\{x_i'\}_{i=1}^m \subseteq \Ee_{y'}$ such that $\forall \ell <\ell_i$, $\eta_{y'}^{(2\ell)}(x_i')=0$.
\end{itemize}
 Then, $|\Ee_{y'}| = m$, $s^{y'}  =  s^y$,   and $y'\in U \mapsto \Ee_{y'}$ is a continuous function.
\end{prop}
\begin{proof}
Without loss of generality, assume that $\ell_1\leq \ell_2\leq \cdots \leq \ell_m$. Let $b \eqdef  \min_i \abs{\eta^{(2 \ell_i)}(x_i)}>0$. Since letting $L = \max_{i=1}^m \ell_i$, $\max_{\ell\leq L} \norm{\eta_{y'}^{(\ell)} - \eta_y^{(\ell)}}_{L^\infty} \to 0$ as $y'\to y$, for all $\epsilon$, there exists $\delta$ such that for all $y'\in B_\delta(y)$,
\begin{enumerate}
\item[(i)] $ \Ee_{y'}\subseteq \enscond{[x-\epsilon,x+\epsilon]}{x\in \Ee_y}$,
\item[(ii)] for all $x\in B_\epsilon(x_i)$,  $\abs{\eta_{y'}^{(2\ell_i)}(x) }>b/2$  and  $\abs{\eta_{y'}^{(2\ell)}(x)} <b/8$  for all $\ell<\ell_i$.
\item[(iii)] $\abs{\eta_{y'}(x)} <1$ for all $x\not\in \cup_i B_\epsilon(x_i)$, 
\item[(iv)]  for all $x\in B_\epsilon(x_i)$ such that $\abs{\eta_{y'}(x)} = 1$, we have $\eta_{y'}(x) = s^y_i$.
\end{enumerate}
Suppose that $\ell_m = \ell_{m-1} = \cdots =\ell_{m-r}$. By assumption, $\eta_{y'}^{(2\ell)}(x_m') = 0$ for all $\ell<\ell_m$. Suppose that $x_{m}'\in B_\epsilon(x_i)$ for some $i<m-r$. Then, $\ell_i < \ell_m$ and $\eta_{y'}^{(2k_i)}(x_m') = 0$ by assumption. However, this contradicts (ii) above, since $\abs{\eta_{y'}^{(2 k_i)}(x_m') }>b/2$.
So, (up to a re-ordering of the points $\{x_{i}'\}_{i=m-r}^m$) we may assume that $x_m'\in B_\epsilon(x_m)$,  $\eta_{y'}(x_m')  = s_m$, $\eta_{y'}^{(\ell)}(x_m')  = 0$ for $\ell=1,\ldots, 2\ell_m - 1$ and $\abs{\eta_{y'}^{(2\ell_m)}(x)}>b/2$ for all $x\in B_\epsilon(x_m)$. Therefore, $\abs{\eta_{y'}(x)}<1$ for all $x\in B_\epsilon(x_m) \setminus \{x_m'\}$. 
By repeating this argument, it follows that for all $i = 1,\ldots, m$, $x_i'\in B_\epsilon(x_i)$ and  $\abs{\eta_{y'}(x)}<1$ for all $x\in B_\epsilon(x_i) \setminus \{x_i'\}$ and hence, $\abs{\Ee_{y'}} = m$.

\end{proof}

%
%
%

\begin{lem}\label{lem:fourier_inj}
Let $y\in \RR^n$.
Suppose that the extended support $\Ee_y$ is discrete, then $\mes_{\beta,\Ee_y}$ with
$$
\beta = \Phi_{\Ee_y}^\dagger(y-\lambda (\Phi_{\Ee_y}^*)^\dagger s_y), \qwhereq s_y \eqdef (\eta_y)_{\restriction_{\Ee_y}}
$$
is the unique solution to \eqref{eq:blasso}.
\end{lem}
\begin{proof}
Since $\Ee_y = \ens{x_j}_{j=1}^m$ is discrete, it is of cardinality  at most $m\leq f_c$, moreover, $\Phi_{\Ee_y}$ is an injective matrix since by considering $\Psi$ from \eqref{equiv_complex}, $\Psi_{\Ee_y}$ is the matrix with $m$ columns of the form $(e^{-2\mathrm{i}\pi \ell x_j})_{\ell=-f_c}^{f_c}$. This is a Vandermonde matrix of size $n\times n$ restricted to $m\leq f_c<n$ columns, and is therefore both $\Phi_{\Ee_y}$ and $\Psi_{\Ee_y}$ are injective. Finally, since $\mathrm{ker}(\Phi_{\Ee_y}) = \ens{0}$, the formula for $\beta$ follows by Lemma \ref{eq:lassosol}.  
\end{proof}

With Proposition \ref{prop:cont_ext} and Lemma \ref{lem:fourier_inj} in mind, we now modify the set $\Hh$ from \eqref{eq:setH-general} such that it is the boundary of sets for which the number of vanishing derivatives at each point of the extended support remains constant.
Define the set
\begin{align*}
\Hh \eqdef \bigcup_{m=1}^{f_c} \bigcup \enscond{
\mathrm{Bd} \Bigg( \Pi_Y \Qq_{(m,\sigma, I, \bell)} \Bigg)  
}{\sigma= \{-1,1\}^m, I\subseteq [m], 
\bell \in\NN^m, \sum_{i=1}^m \bell_i\leq n 
}
\end{align*}
where
\begin{equation}
\begin{split}
\Qq_{(m,\sigma, I, \bell)} \eqdef  \Bigg\{ (y, \Ee,\beta)&\in \RR^n \times\TT^m\times \RR^m ;\eta\eqdef \Phi^*( y - \Phi_\Ee \beta) ,  M\eqdef \Phi_{\Ee}, \Ee\eqdef \{x_i\}_{i=1}^m\\
&\eta \in \partial \abs{\mes_{\beta,\Ee}}_{TV},  \mathrm{Rank}(M) =m\\
& 
\pa{ M^\dagger (y - (M^*)^\dagger \lambda \sigma )   }_I = 0_I  \\
& \eta^{(2j)}(x_i) = 0, \; j \in[ \bell_i-1], i\in [ m] \Bigg\}.
\end{split}
\end{equation}
We also define for $i\in\{+1,-1\}$, the following sets
\begin{align*}
\Gg^+ \eqdef \mathrm{Bd}\pa{\Pi_Y\enscond{(y,\beta,\Ee)\in \RR^n \times \RR_{\geq 0}^n\times \TT^n}{ y - \Phi_\Ee \beta  =  \lambda \delta_1 }}
\\
\Gg^- \eqdef \mathrm{Bd}\pa{\Pi_Y\enscond{(y,\beta,\Ee)\in \RR^n \times \RR_{\leq 0}^n\times \TT^n}{ y- \Phi_\Ee \beta  = - \lambda \delta_1 }}
\end{align*}
where $\delta_1$ be the vector of length $n$  with first entry equal to one, and all other entries equal to zero.

The set $\Kk \eqdef \Hh\cup\Gg^{+}\cup \Gg^{-}$ can be shown to be a set of measure zero (See Proposition \ref{prop:K_measurezero}).  The following theorem is the main result of this section and  shows that the divergence can be computed for all $y\not\in \Kk$.

\begin{thm}\label{thm:main2}
For all $y\not\in \Kk$, 
\begin{align*}
\dive(\hat \mu)(y) = \begin{cases}
0 & \qifq \Ee_y = \emptyset, \\
2f_c+1 & \qifq  \Ee_y = \TT, \\
2k - \nu & \qifq \Ee_y \text{ is a discrete set.}
\end{cases}
\end{align*}
where in the case of $\Ee_y$ is discrete, $\nu = -\lambda \sum_{j=1}^k \frac{1}{\beta_j}\eta_y''(x_j) v_j \geq 0$,  where $\Aa\eqdef \{x_j\}_{j=1}^k$ and $\beta$ having all non-zero entries is such that $\mes_{\beta,\Aa}$ is the solution to \eqref{eq:blasso},
and $
v_j = (M^{-1})_{s+j}
$ where $M$ is as in \eqref{eq:M} with $\Xx = \Aa$.
\end{thm}

\begin{rem}
We defer the proof of Theorem \ref{thm:main2} to Appendix \ref{app-proof-thm-main}, since its proof is similar to that of Theorem~\ref{thm:cont_ext_supp1}. We however mention two key
 properties of the Fourier setting which allow us to relax the assumptions in this result:
\begin{enumerate}
\item The fact that any element of $\Im(\Phi^*)$ is either constant or has finitely many roots, each of which has finite multiplicity. This ensures  that the extended support moves in a continuous manner.
\item $\Gamma_\Aa^*\Gamma_\Aa$ is invertible, which ensures that we can invoke the implicit function theorem to conclude that the path $y\mapsto (\beta, \Aa)$ is $\Cc^1$.
\end{enumerate}
\end{rem}

To conclude this section, we prove that the set $\Kk$ is of zero measure. This result follows by simple modificatons of the proof of Proposition \ref{prop:definable}, since in the case of sampling Fourier coefficients, $x\mapsto \phi(x)$ is semi-algebraic and hence definable. For completeness, we present a proof using directly results from semi-algebraic geometry.

\begin{prop}\label{prop:K_measurezero}
The set $\Kk$ is of zero measure.
\end{prop}
\begin{proof}
To prove that $\Hh$ is of zero measure, it is sufficient to show that $\dim(\mathrm{Bd}\pa{\Pi_Y(\Qq_{(m,\sigma,I,\bell)})}<n$, since the countable union of zero measure sets is of zero measure.
To prove this, it is sufficient to show that $\Qq_{(m,\sigma,I,\bell)}$ is a semi-algebraic set. Then, by the  Tarski-Seidenberg principle \cite[Thm. 2.3]{coste2000introduction}, $\Pi_Y (\Qq_{(m,\sigma,I,\bell)})$ is a semi-algebraic set and is of dimension at most $n$. Finally, by Theorem 3.22 in \cite{coste1999introduction}, we have $$\dim(\mathrm{Bd}\pa{\Pi_Y(\Qq_{(m,\sigma,I,\bell)})} < \dim\pa{\Pi_Y(\Qq_{(m,\sigma,I,\bell)})} \leq n.
$$
To see that $\Qq_{(m,\sigma,I,\bell)} $ is a semi-algebraic set \footnote{A semi-algebraic set in $\RR^n$ is a set of vectors in $\RR^n$ satisfying a boolean combination of polynomial equations.  Moreover, sine and cosine are semi-algebraic functions. 
}, note that
\begin{align*}
\Qq_{(m,\sigma,I,\bell)} =& A \cap B
\end{align*}
where
\begin{align*}
A\eqdef \enscond{(y,\Ee,\be)}{ \substack{\eta =\frac{1}{\la} \Phi^*(y-\Phi_\Ee \be), \norm{\eta}_\infty\leq 1, \eta_{\restriction \Ee} = \sign(\be), \\
\eta^{(2j)}(x_i)= 0, j\in [\bell_i-1], i\in [m]}}
\end{align*}
and 
$$
	B\eqdef  \enscond{(y,\Ee,\be)}{M = \Phi_\Ee, \; \pa{ M^\dagger (y - (M^*)^\dagger \lambda \sigma )   }_I = 0_I }
$$
Note that $M \mapsto M^*$, $M\mapsto M^\dagger$ and $M\mapsto M_S$ are semi-algebraic mappings and the composition of semi-algebraic mappings is semi-algebraic. Also, $\Ee \to \Phi_\Ee$ is semi-algebraic. Therefore, $B$ is a semi-algebraic set.

For the set $A$, $f_x : (y,\Ee,\be) \mapsto (\Phi^* (y- \Phi_\Ee \be))(x)$ is a semi-algebraic mapping. The constraint $\norm{\eta}_\infty\leq 1$ is
$$
\forall x\quad  -1 \leq f_x(y,\Ee,\be) \leq 1 \iff \neg\ens{ \exists x \; \pa{f_x(y,\Ee,\be) < - 1 \text{ or }   f_x(y,\Ee,\be) >1} }.
$$
which is a semi-algebraic constraint \cite[page 28]{coste2000introduction}. Finally, the derivatives of semi-algebraic mappings are semi-algebraic \cite[Ex 2.10]{coste2000introduction}  and $f(\be) = \sign(\be)$ is also semi-algebraic.

Finally, again by the Tarski-Seidenberg principle, ${\Pi_Y\enscond{(y,a,\Ee)\in \CC^n \times \RR_{\pm}^n\times \TT^n}{ y - \Phi_\Ee a  =\pm  \lambda \delta_1 }}$ are semi-algebraic sets, and their boundary is of measure strictly smaller than $n$, so both $\Gg^+$ and $\Gg^-$ are of zero measure.

\end{proof}

\section{Remarks on positivity constraint}


The results of the previous sections can be extended to other sparsity-enforcing convex optimization problems over the space of measure. We present here the extension to the following regression problem under positivity constraints:
\begin{equation}\label{eq:pos_constr1}
\inf_{\mes\in\Mm(\Om)} \frac{1}{2}\norm{\Phi \mes - y}_2^2 \text{ subject to }  \mes\geq 0 \tag{$\Pp_+(y)$}
\end{equation}

We have the following properties for its Legendre Fenchel dual:
\begin{lem}\label{lem_pos1}
The Legendre-Fenchel dual of \eqref{eq:pos_constr1} reads
\begin{equation}\label{eq:dual_pos_constr1}
\sup_{\Phi^* p \geq 0} -\frac{1}{2}\norm{p+y}_2^2 +\frac12 \norm{y}^2.  \tag{$\Dd_+(y)$}
\end{equation}
Moreover,
\begin{itemize}
\item[(i)] strong duality holds with $\inf \eqref{eq:pos_constr1} = \sup \eqref{eq:dual_pos_constr1}$.
\item[(ii)]  If $\mes_y$ and $p_y$ are respectively  solutions to \eqref{eq:pos_constr1} and \eqref{eq:dual_pos_constr1}, then letting $p_y = \Phi \mes_y - y$, $\mathrm{Supp}(\mes_y) \subseteq \enscond{x\in\Om}{\Phi^* p_y(x) = 0}$.
\item[(iii)]  If $\mes_y= \sum_{j=1}^k \beta_j \delta_{x_j}$ with $\beta_j\geq 0$  and $p_y$ are respectively  solutions to \eqref{eq:pos_constr1} and \eqref{eq:dual_pos_constr1} if and only if
$$
\Phi \mes_y = p_y+y, \quad \Phi^* p_y \geq 0 \qandq  (\Phi^* p_y)(x_j) = 0.
$$
\end{itemize}
\end{lem}
The proof of this lemma can be found in Appendix \ref{app:pos_lem}.


For \eqref{eq:pos_constr1}, we see that assuming that there is a neighbourhood $U$ of $y$ such that the solution consists of $k$ spikes whose positions and amplitudes follow a differentiable path, then for $y\in U$, the solutions $\mes_y \eqdef \mes_{(\be_y,\Xx_y)}$ satisfy
$$
\Phi_{\Xx_y}^* \Phi_{\Xx_y} \be = \Phi_{\Xx_y}^* y \qandq (\Phi_{\Xx_y}^{(1)})^* \Phi_{\Xx_y}\be = (\Phi_{\Xx_y}^{(1)})^* y
$$
and differentiating this leads to the same expression for the divergence of $\hat \mu(y) = \Phi \mes_y$. It is also straightforward to extend the results of Sections \ref{sec:divergence} and  \ref{sec:fourier} for the problem \eqref{eq:pos_constr1} (we simply replace the sign vector $s_y$ with the zero vector).


\section{Numerical Experiments}\label{sec:numerics}


Although it is not the purpose of this paper, let us mention some works on  devising efficient numerical scheme to solve exactly or approximately the infinite-dimensional optimization problem~\eqref{eq:blasso}. For Fourier measurements, it is possible to use method from polynomial optimization and sum-of-squares semi-definite programming relaxation~\cite{candes-towards2013,azais-spike2014,de2016exact}. For the more general problem, one can use greedy-type methods, which are extensions of the celebrated Frank-Wolfe method~\cite{bredies-inverse2013,boyd2017alternating,denoyelle2019sliding}, see also~\cite{chizat2018global} for a non-convex solver with global convergence guarantees.

In the following, we  numerical experiments  to validate our theoretical findings.   The experiments are computed using the sliding Frank-Wolfe method introduced in \cite{denoyelle2019sliding}.

\subsection{Stein's unbiased risk estimate (SURE)}
\label{sec-sure}

Given samples $y\sim \Nn(\mu,\sigma^2 \Id_d)$, let $\hat \mu : \RR^n \to \RR^n$ be an estimate of $\mu$ from $y$. Then,  a quick computation reveals that the risk can be expressed as
\begin{align*}
	R &\eqdef \EE\norm{\hat \mu - \mu}^2 = \EE \norm{\mu - y + y - \hat \mu}\\
	&= - n\sigma^2 + \EE\norm{y - \hat \mu}^2 + 2 \sum_{i=1}^n \mathrm{Cov}(y_i, \hat \mu_i).
\end{align*}
So, Stein's lemma \cite{stein1981estimation} gives an explicit estimate of the risk  in the case where the estimator $\hat \mu$ is  almost differentiable, that is,
$$
	\sure(\hat \mu(y)) \eqdef -n \sigma^2 + \norm{y - \hat \mu}^2 +  2\sigma^2 \dive(\hat \mu)(y).
$$ This estimate is referred to as the Stein's unbiased risk estimate. 

Our main result presents a closed form expression for $\dive(\hat \mu)(y)$ and shows this to be at most the number of recovered parameters (and smaller given conditions on the curvature of the dual certificate). In this section, we demonstrate the importance of our result by comparing the SURE against the estimate obtained if $\dive(\hat \mu)(y)$ was taken to be the number of recovered parameters
$$
	\sure_{\mathrm{param}}(\hat \mu(y)) \eqdef -n \sigma^2 + \norm{y - \hat \mu}^2 +  2\sigma^2 P,
$$
where $P = k(d+1)$ where $k$ is the number of Diracs in the solution  $\mes_{\beta,\Xx}$ to~\eqref{eq:blasso} of smallest support.


For some fixed $\mu$, we generate $K=200$ instances  $(y_i)_{i=1}^K$ in accordance to the Gaussian distribution  with mean $\mu$ and standard deviation $\sigma$. That is, $y_i \sim \Nn(\mu, \sigma^2 \Id)$ for $i\in [K]$. Then, given some $\lambda>0$, for each $y_i$, we solve \eqref{eq:blasso} using the sliding Frank-Wolfe algorithm~\cite{denoyelle2019sliding} to obtain $\mes_{\beta_{y_i},\Xx_{y_i}}$. 
Note that under some non-degeneracy condition, it has been proved that this algorithm converges in a finite number of steps (thus computing a discrete sparse solution) and it can be thus used a efficient scheme to have access to $\mes_{\beta_{y_i},\Xx_{y_i}}$ (since one can check a posteriori that $\Gamma_\Xx$ is injective and thus the solution is the unique one). 
Let $\hat \mu_\lambda(y_i)  \eqdef  \Phi \mes_{\beta_{y_i},\Xx_{y_i}}$. 
We then compute the SURE using $\dive(\hat \mu)(y_i)$ as derived in Theorem \ref{thm:main2}:
$$
	\sure(\hat \mu_\la(y_i)) = -n \sigma^2 + \norm{y_i - \hat \mu_i}^2 +  2\sigma^2 \dive(\hat \mu)(y_i)
$$
and the SURE   where   $\dive(\hat \mu)(y_i)$ is replace with the number of recovered parameters $P_i = k_i (d+1)$ where $k_i$ is the length of $\beta_{y_i}$:
$$
	\sure_{\mathrm{param}}(\hat \mu_\la(y_i))  \eqdef -n \sigma^2 + \norm{y_i - \hat \mu_i}^2 +  2\sigma^2 P_i.
$$

We carry out this numerical experiment for the two cases already mentioned in Section~\ref{sec-curse}
\begin{enumerate}
\item The sampling of Fourier coefficients in dimension 1, where $$\Phi \mes= \frac{1}{\sqrt{2f_c+1}} \pa{ \int e^{-2 \imath\pi k t} \mathrm{d}\mes(t)}_{\abs{k}\leq f_c, k\in\ZZ}$$ with $f_c = 10$ and set $\sigma= 0.01$. We also let $\mu = \Phi \mes_{\beta,\Xx}$ be generated by 3 spikes, with $\beta = [2,-4.5,4]$ and $\Xx = [0.1,0.6,0.9]$.
For convenience of implementation, we use the complex exponential formulation, which is equivalent to the sine and cosine formulation as mentioned in \eqref{equiv_complex}.

\item Learning a two-layer neural network. Given data $(a_j,y_{0,j})$ for $j=1,\ldots,n$ with $a_j\in \RR^d$ and $y_{0,j}\in \RR$, we use a normalized version of the parameterization explained in Section~\ref{sec-curse}, namely
\eq{
	\phi(x) = \frac{\tilde\phi(x)}{\norm{\tilde\phi(x)}} 
	\qwhereq
	\tilde\phi(x) = ( \xi( \dotp{x}{a_j} ) )_{j=1}^n \in \RR^n
}
and where $\xi(r) = \max(r,0)$ is the ReLu non-linearity.
Note that if $\mes= \sum_i \be_i \delta_{x_i}$, then
$$
	\Phi \mes = \pa{ \sum_i  \frac{\be_i}{ \norm{\tilde\phi(x_i)} }  \max\pa{0,\dotp{a_j}{x_i}} }_{j=1}^n
$$
so $(x_i)_i$ and $(\be_i)_i$ respectively represent the parameters ($n$ neurons) of the hidden and output layers of the trained neural network. 
The formulation of a two-layer neural network using sparse measure was introduced in \cite{bach2017breaking}, see also~\cite{chizat2018global}. 
In our experiment, we choose $\sigma= 0.05$, $d=50$ and $n=500$, and $a_i \overset{iid}{\sim} \Nn(0,\Id_d)$. We also fix $\mu \eqdef \Phi \mes_{\beta,\Xx}$, with $\beta_j\in \RR^3$ and $\Xx\in \RR^{3d}$ where $\beta_j \in \Nn(0, 1)$ and $ \Xx_j\in \Nn(0,10^3 \Id_d)$, 
\end{enumerate}

Figure \ref{fig:sure-fourier} and \ref{fig:sure-nn} show plots of the average SURE values for different values of $\lambda$ and the mean squared error:
\begin{align*}
\mathrm{MSE} &\eqdef  \frac{1}{K} \sum_{i=1}^K  \norm{\mu -\hat \mu(y_i)}^2\\
\sure &\eqdef -n \sigma^2 +  \frac{1}{K} \sum_{i=1}^K  \pa{\norm{y_i - \hat \mu(y_i)}^2 +  2\sigma^2 \dive(\hat \mu)(y_i) }
\\
\sure_{\mathrm{param}} &\eqdef -n \sigma^2 +\frac{1}{K} \sum_{i=1}^K\pa{\norm{y _i- \hat \mu(y_i)}^2 +  2\sigma^2 P_i},
\end{align*}
where $P_i$ in $\sure_{\mathrm{param}} $ is the number of recovered parameters for the $i^{th}$ run.

\begin{figure}
\begin{center}
\includegraphics[width=0.48\textwidth]{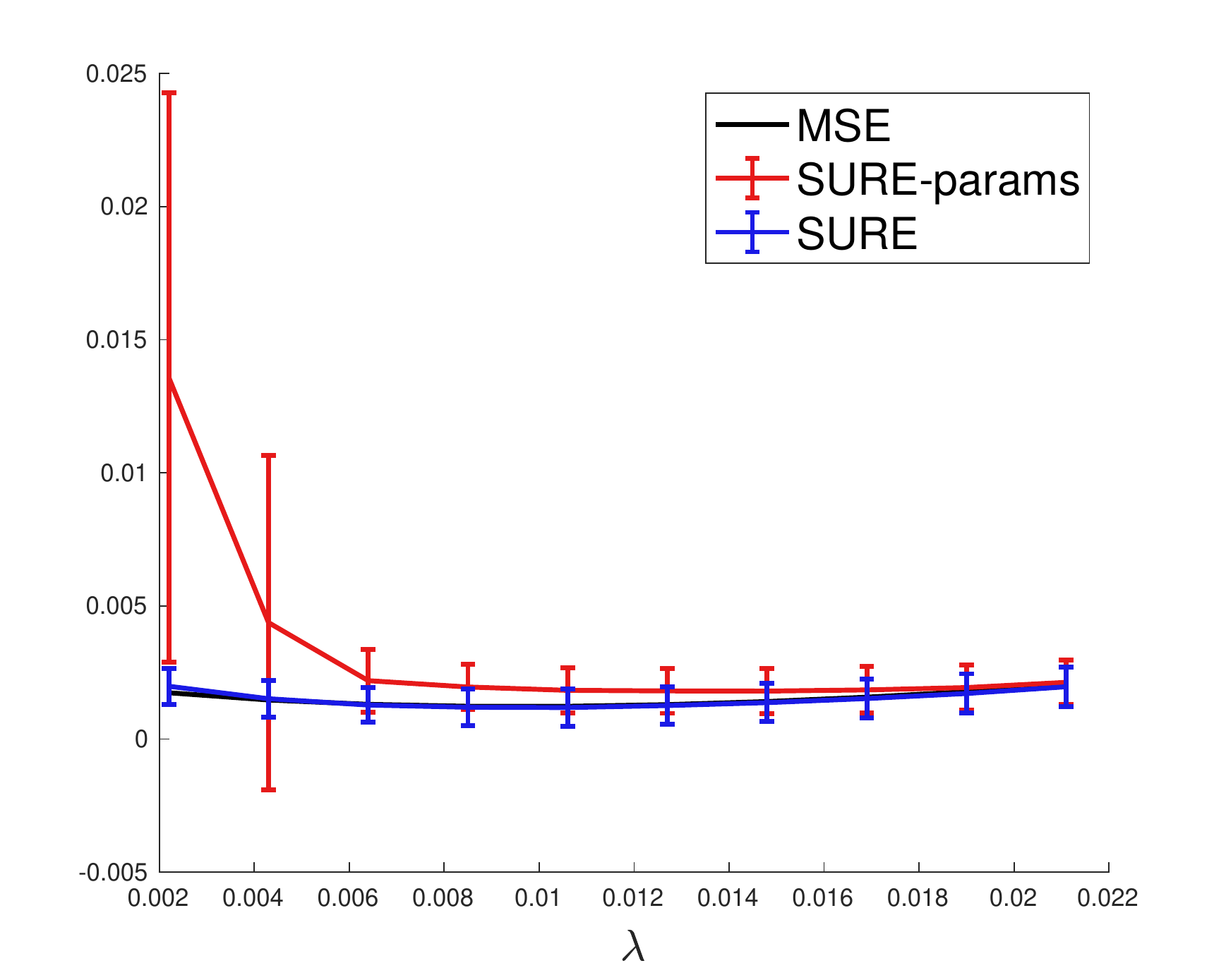}
\includegraphics[width=0.48\textwidth]{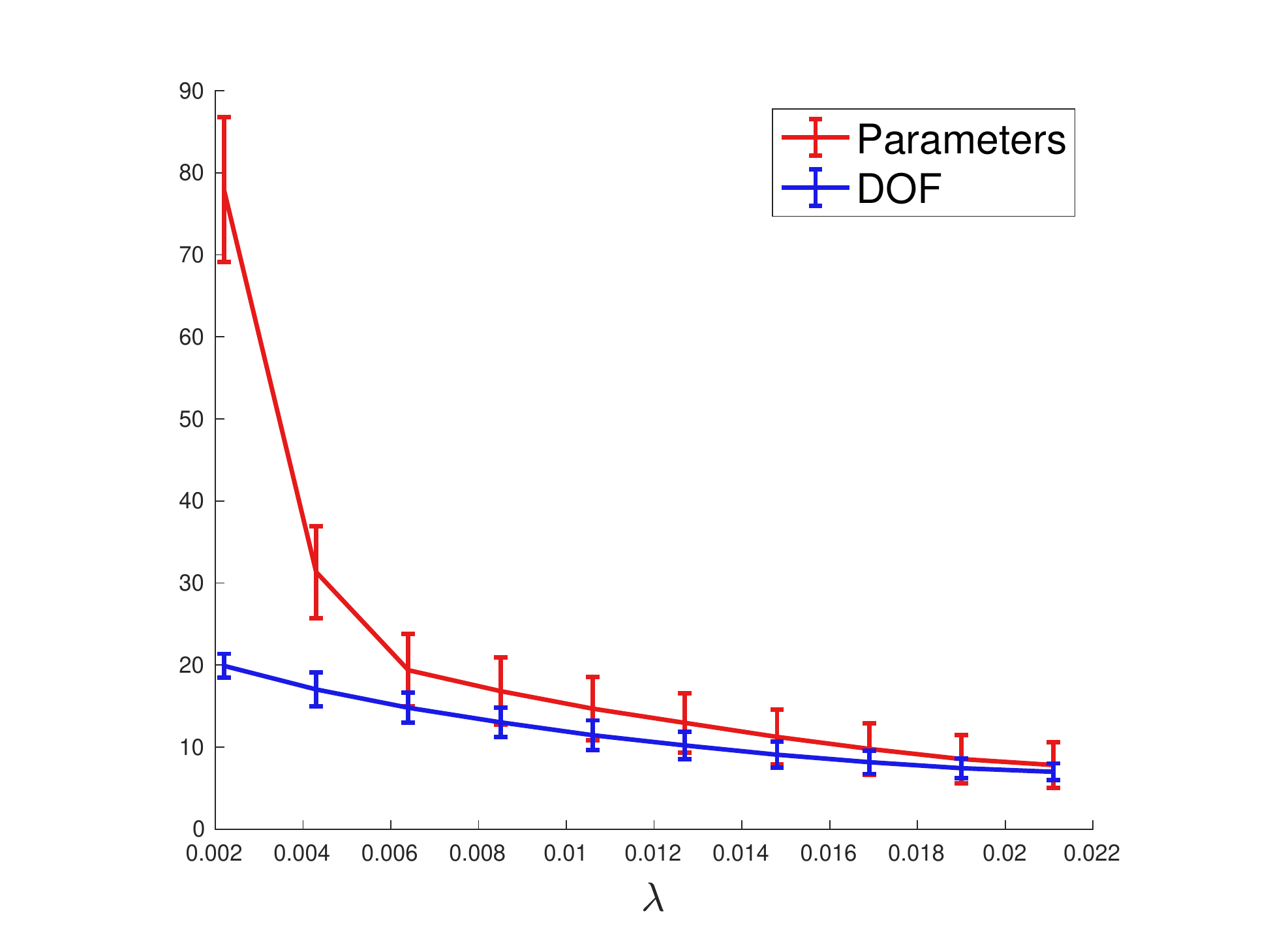}
\end{center}
\caption{Comparison of the SURE and DOF for Fourier sampling. Left:  the true SURE versus the estimate computed using the number of recovered parameters. Right: the DOF versus the number of recovered parameters. The error bars show the standard deviation.  \label{fig:sure-fourier}}
\end{figure}

\begin{figure}
\begin{center}
\includegraphics[width=0.48\textwidth]{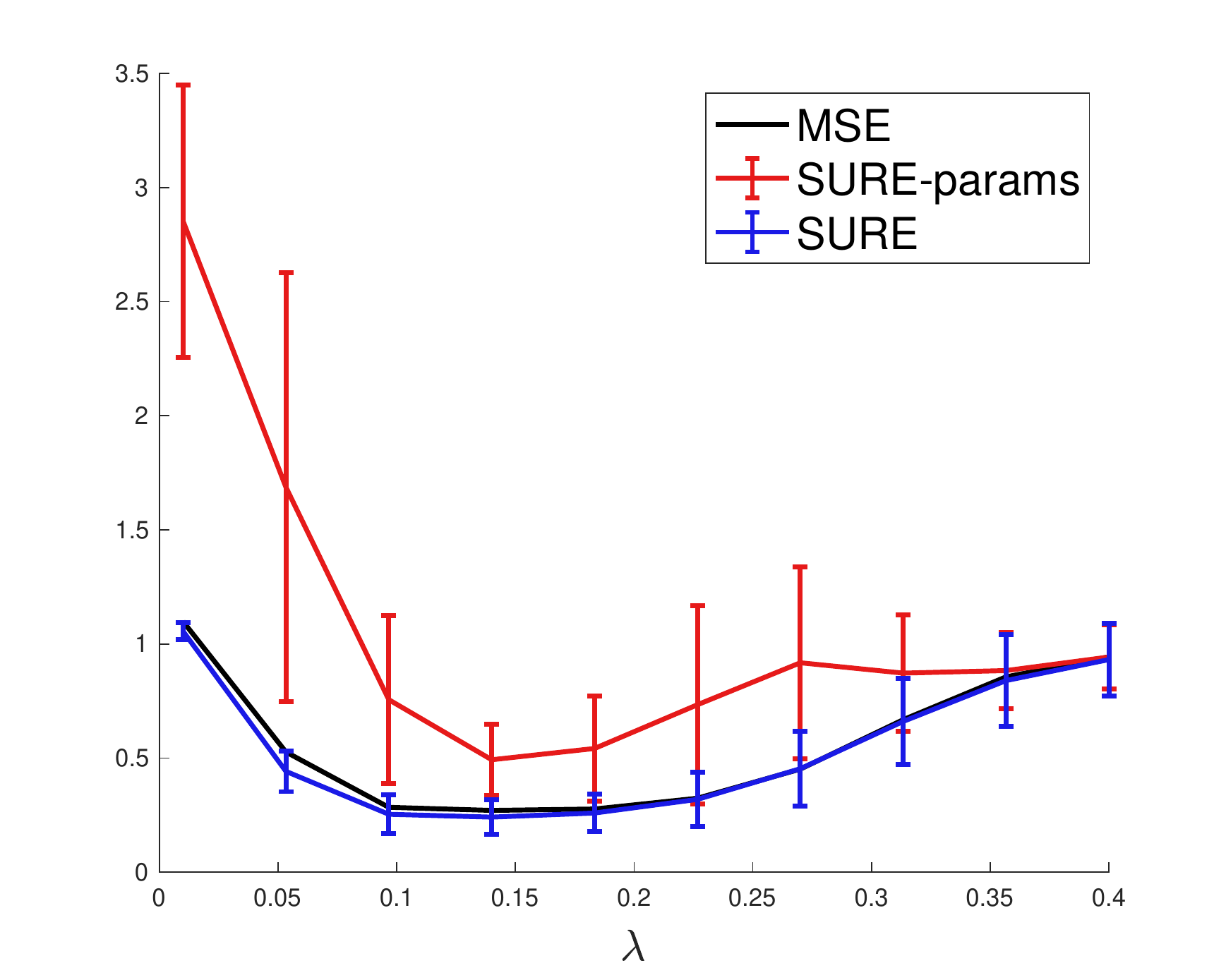}
\includegraphics[width=0.48\textwidth]{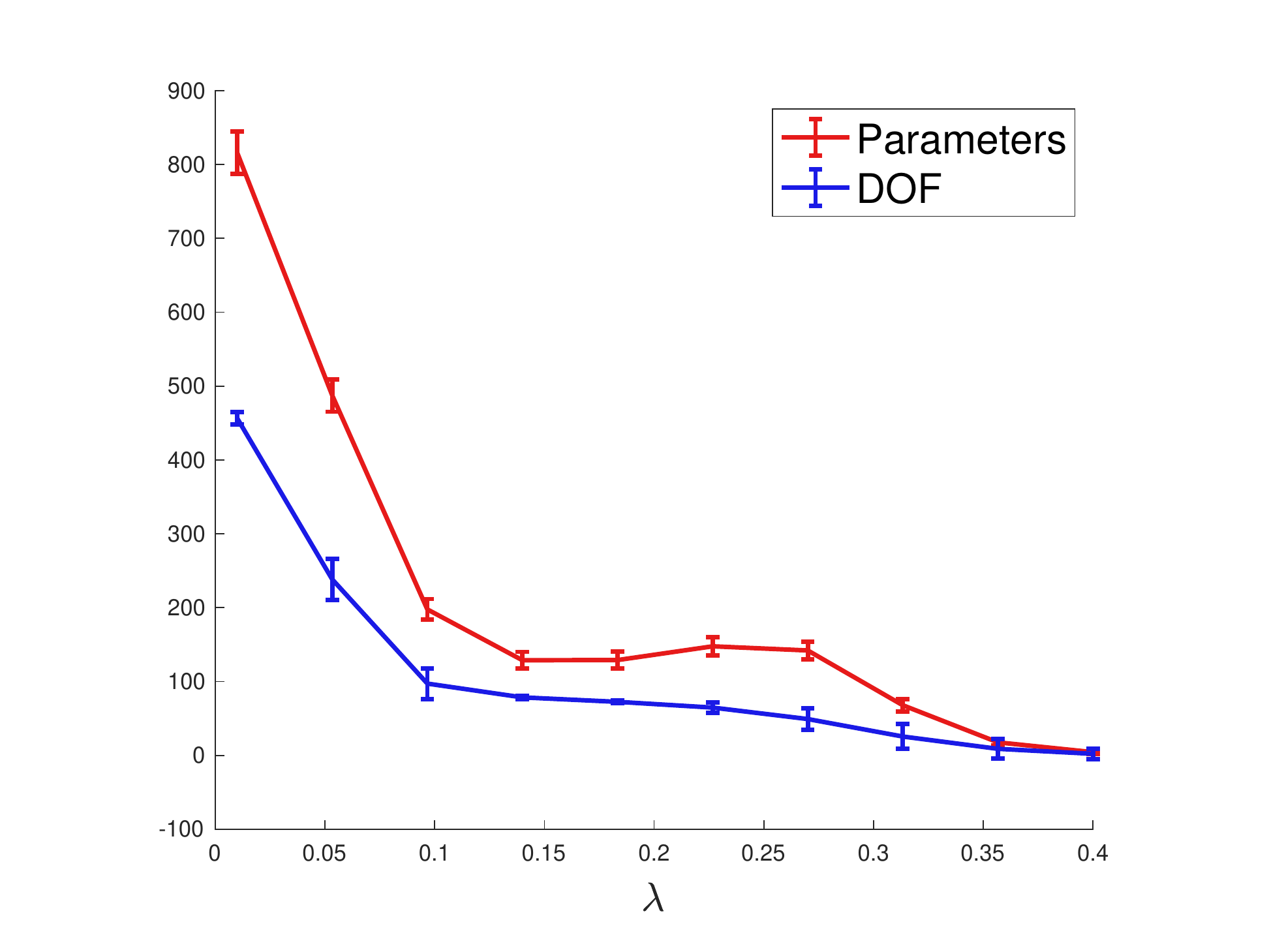}
\end{center}
\caption{Comparison of the SURE and DOF for a 2-layers neural network in dimension 50.  Left:  the true SURE versus the estimate computed using the number of recovered parameters. Right: the DOF versus the number of recovered parameters.  The error bars show the standard deviation. \label{fig:sure-nn}}
\end{figure}



\section{Conclusion}

In this paper, we have derived a formula for the degrees of freedom of sparse off-the-grid methods, and in particular for the Blasso and regression under positivity constraints. 
These results highlight the fact that $k$-sparse solutions of variational problems over $\RR^d$ have in general much fewer intrinsic parameters than the number $k(d+1)$ of free variables involved and that solving discretized problems typically tends to over-estimate the prediction risk. Controlling this gap is an interesting avenue for future works. This gap is primarily governed by the curvature induced by the underlying continuous model $\phi(x)$, but its exact value depends in a complicated way on relative positions between the estimated Dirac masses. 

\appendix

\section{Comment on the trace formula}\label{sec:comment-trace}

We aim at showing  $T\eqdef \tr\pa{ \Pi_{\Im(\Gamma_\Xx)} \begin{pmatrix}
0 & 0\\
0 & \diag(Z) D_a^{-1}
\end{pmatrix} M^{-1}    } \geq 0$.
Writing $\tilde D \eqdef  \begin{pmatrix}
0 & 0\\
0 & \diag(Z) D_\be^{-1}
\end{pmatrix}$, we have
\begin{align*}
(\Gamma_\Xx^* \Gamma_\Xx)^\dagger \tilde D M^{-1}   \Gamma_\Xx^* \Gamma_\Xx
&= (\Gamma_\Xx^* \Gamma_\Xx)^\dagger \tilde D M^{-1}  \pa{ M -\tilde D} \\
&=  (\Gamma_\Xx^* \Gamma_\Xx)^\dagger \tilde D^{\frac12}  \pa{ \Id  - \tilde D^{\frac12}   M^{-1}  \tilde D^{\frac12}}  \tilde D^{\frac12}
\end{align*}
and note that $M \succeq \tilde D$ which implies that $M^{-1} \preceq \tilde D^{\dagger}$ which implies that $ \tilde D^{\frac12}   M^{-1}  \tilde D^{\frac12} \preceq  \tilde D^{\frac12}    \tilde D^{\dagger} \tilde D^{\frac12}  \preceq \Id$. Therefore, $B \eqdef \tilde D^{\frac12}\pa{\Id  - \tilde D^{\frac12}   M^{-1}  \tilde D^{\frac12}}  \tilde D^{\frac12}\succeq 0$ and hence, $T = \tr\pa{\pa{(\Gamma_\Xx^* \Gamma_\Xx)^{\dagger}}^\frac12   B \pa{(\Gamma_\Xx^* \Gamma_\Xx)^{\dagger}}^\frac12  } \geq 0$ .

\section{O-minimal geometry and negligibility of the set $\Hh$} \label{sec:ominimal}

We first recall some facts about definable sets \cite{coste1999introduction,vaiter2017degrees}:
\begin{itemize}
\item The addition, multiplications and composition of of definable  functions are definable.
\item The Jacobian of a differentiable and definable function is definable.
\item Note that if $f:\RR^n\to \RR^m$ is a definable function,  then for all  definable subsets $I$ of $ \RR^m$, 
$$
\enscond{x}{f(x) \in I} = \Pi_{(n+m),n}\pa{\enscond{(x,z)}{ f(x) = z  } \cap (\Om \times I)}
$$
is definable, where $ \Pi_{(n+m),n}$ is the projection onto the first $n$ variables. In particular, $\enscond{x}{f(x) =y}$ is definable.

\item If $f: A\times B \to C$ is definable and $I$ is definable, then the following are definable:
\begin{align*}
&\enscond{y\in B}{\exists x, \; f(x,y) \in I} = \Pi_Y \enscond{(x,y)}{f(x,y) \in I}\\
&\enscond{y\in B}{\forall x, \;  f(x,y) \in I}  = B\setminus  \enscond{y\in B}{\exists x, \;  f(x,y) \in C\setminus I},
\end{align*}
since definable sets are stable in taking complements, and projections.

\item Note that $f_1 = a\mapsto \abs{a}$ and $f_2 = a\mapsto\sign(a)$ are semi-algebraic
$$
\Gg_{f_1} = \enscond{(a,b) }{ \pa{a+b = 0 \text{ or } a-b = 0} \text{ and } b>0}
$$
$$
\Gg_{f_2} = \enscond{(a,b)}{ a\cdot b = \abs{a}}.
$$ 
\item Given $M\in \RR^{n\times m}$, $M\mapsto M^*$ is definable (in fact it is algebraic), and $M\mapsto M^\dagger$ is also algebraic and hence definable, since by definition of the pseudoinverse, its graph
\begin{align*}
&\enscond{(M,A)}{A = M^\dagger}=\\
&
\quad \enscond{(M,A)}{ MAM = M, AMA = A, (MA)^* = MA, (AM)^* = AM}
\end{align*}
is an algebraic set.
\item Given $s\leq \min\{m,n\}$, $ \enscond{A\in \RR^{m\times n}}{\rank(A) = s}$ is a semi-algebraic set since $\rank(A) = s$ if and only if
\begin{align*}
A\in &\Aa \eqdef \pa{\cap_{I \subset [n], \abs{I} = s} \ens{\not \exists c\in \RR^s \text{ such that } A_I c = 0}}\\
&\cap  \pa{\cup_{\abs{I} = s+1} \ens{\exists c\in \RR^{s+1} \text{ such that } A_I c = 0}}.
\end{align*}
Note that $\Aa$ is made up of unions and intersection of finitely many sets, each of which is semi-algebraic since they are defined using first order formulas.
\end{itemize}

\begin{prop}\label{prop:definable}
Assume that $x\mapsto \phi(x)$ is a definable function. Then, $\Hh$ is of Lebesgue measure zero.
\end{prop}
\begin{proof}
It is enough to show that $\mathrm{Bd}(\Pi_Y (\Qq_{k,S,I,\sigma}))$ is of Lebesgue zero measure, in particular, we need to show that $\dim\pa{\mathrm{Bd}(\Pi_Y (\Qq_{k,S,I,\sigma}))} <n$.

Let $\Qq \eqdef \Qq_{k,S,\sigma,I}$. First note  that since $x\mapsto \phi(x)$ is definable, $\Ee \eqdef \{x_j\}_{j=1}^k \mapsto \Phi_\Ee = (\phi(x_j))_j\in \RR^{n\times k}$ is also definable.
Also,t $f (x, y,a,\Ee) \mapsto \frac{1}{\la} \dotp{\phi(x)}{y - \Phi_\Ee a}$ is a definable function. Define the sets
\begin{align*}
\Aa_1 \eqdef  \enscond{(y,a,\Ee)}{  
\Ee = \{x_j\}_{j=1}^k, \; \forall j\in[k], \;  f({x_j},y,a,\Ee) = \sign(a_j)  , \qandq  \forall  x \abs{f(x,y,a,\Ee)} \leq 1 }.\\
\Aa_2 \eqdef \enscond{(y,\Ee)}{ M = \Phi_\Ee, \;  \rank(\Phi_\Ee) = \abs{S} \; ((M^\dagger + M_S^\dagger M_{S^c} (M)^\dagger_{(S^c,\cdot)}   )  (y-(M^*)^\dagger \lambda \sigma  )_I = 0_I}
\end{align*}
These are both definable sets since the constraints are definable by the properties listed above, and hence,
$$
\Pi_Y(\Qq) = \Pi_Y( \Pi_{Y,\Ee}\Aa_1 \cap \Aa_2),\qwhereq \Pi_{Y,\Ee}:(y,a,\Ee) \mapsto (y,\Ee)
$$
is also definable.
Finally, since $\Pi_Y(\Qq)  \subseteq \RR^n$ is definable, we have $\mathrm{dim}(\Pi_Y(\Qq)) \leq n$
$$
\mathrm{Bd}\pa{\Pi_Y(\Qq) } < \mathrm{dim}(\Pi_Y(\Qq)) \leq n. 
$$
\end{proof}

This results holds for instance for Fourier measurements and neural network with a Relu activation (which leads to semi-algebraic sets) and for deconvolution using a Gaussian point spread function (since the exponential function is definable in an o-minimal structure~\cite{wilkie1996model}).

\section{Full rank of Fourier matrix}\label{app:fullrank}

If the extended support is not $\TT$, then it consists of at most $k\leq f_c$ points, so $2k<n$. In this case, by observation \eqref{equiv_complex}, $\Gamma_{\Xx}$ is injective provided that
$\tilde \Gamma_\Xx$, defined  below, is injective:
$$
\tilde \Gamma_\Xx \eqdef
\begin{pmatrix}
u_1^{-n} &u_2^{-n}& \cdots & u_k^{-n}&(-2\pi n) u_1^{-n} &(-2\pi n) u_2^{-n}& \cdots &(-2\pi n)  u_k^{-n}\\
\vdots\\
u_1^{-1} &u_2^{-1}& \cdots & u_k^{-1}&(-2\pi) u_1^{-1} &(-2\pi ) u_2^{-1}& \cdots &(-2\pi)  u_k^{-1}\\
1 &1& \cdots & 1&0&0&\cdots & 0\\
u_1 &u_2& \cdots & u_k &(2\pi) u_1 &(2\pi ) u_2& \cdots &(2\pi)  u_k^{-1}\\
\vdots\\
u_1^{n} &u_2^{n}& \cdots & u_k^{n}&(2\pi n) u_1^{n} &(2\pi n) u_2^{n}& \cdots &(2\pi n)  u_k^{n}
\end{pmatrix}
$$
where $u_j = e^{2\imath\pi x_j}$.  We now carry out row-echelon and column-echelon operations to show that $\tilde \Gamma_\Xx$ is indeed full rank.

%

%

After multiplying the last $k$ columns by $1/(2\pi)$, then for all $j\in[k]$, multiply  column $j$ and $2j$ (corresponding to $u_j$) by $u_j^n$, we obtain:
$$
\begin{pmatrix}
1 &1& \cdots & 1&(-  n)  &(-  n) & \cdots &(-  n)  \\
\vdots\\
u_1^{n-1} &u_2^{n-1}& \cdots & u_k^{n-1}&(- 1) u_1^{n-1} &(-1 ) u_2^{n-1}& \cdots &(-1)  u_k^{n-1}\\
u_1^n &u_2^n & \cdots & u_k^n&0&0&\cdots & 0\\
u_1^{n+1} &u_2^{n+1}& \cdots & u_k^{n+1} & u_1^{n+1} & u_2^{n+1}& \cdots &  u_k^{n+1}\\
\vdots\\
u_1^{2n} &u_2^{2n}& \cdots & u_k^{2n}& n u_1^{2n} &n u_2^{2n}& \cdots & n  u_k^{2n}
\end{pmatrix}
$$

Given a sequence $\{u_j\}_{j=1}^t$ for distinct numbers so that $n>t\geq k$ and $n\geq t+k$, we aim to show that the following matrix is full rank by performing row and column operations.
$$
V^u_{t,k,2n}\eqdef \begin{pmatrix}
1 &1& \cdots & 1&(-  n)  &(-  n) & \cdots &(-  n)  \\
\vdots\\
u_1^{n-1} &u_2^{n-1}& \cdots & u_t^{n-1}&(- 1) u_1^{n-1} &(-1 ) u_2^{n-1}& \cdots &(-1)  u_k^{n-1}\\
u_1^n &u_2^n & \cdots & u_t^n&0&0&\cdots & 0\\
u_1^{n+1} &u_2^{n+1}& \cdots & u_t^{n+1} & u_1^{n+1} & u_2^{n+1}& \cdots &  u_k^{n+1}\\
\vdots\\
u_1^{2n} &u_2^{2n}& \cdots & u_t^{2n}& n u_1^{2n} &n u_2^{2n}& \cdots & n  u_k^{2n}
\end{pmatrix}
$$

For $j=1,\ldots, k$, add $n$ times column $j$ to column $t+j$
$$
\begin{pmatrix}
1 &1& \cdots & 1&0  &0& \cdots &0\\
u_1 & u_2 &\cdots & u_t & u_1 & u_2 &\cdots &u_k \\
\vdots\\
u_1^{n-1} &u_2^{n-1}& \cdots & u_t^{n-1} &(n- 1) u_1^{n-1} &(n-1 ) u_2^{n-1}& \cdots &(n-1)  u_k^{n-1}
\\
u_1^n &u_2^n & \cdots & u_t^n &n u_1^n &nu_2^n &\cdots & n u_k^n
\\
u_1^{n+1} &u_2^{n+1} & \cdots & u_t^{n+1}  &(n+1) u_1^{n+1} & (n+1) u_2^{n+1}& \cdots & (n+1) u_k^{n+1}
\\
\vdots\\
u_1^{2n} &u_2^{2n} & \cdots & u_t^{2n} & 2n u_1^{2n} & 2n u_2^{2n}& \cdots & 2n  u_k^{2n}
\end{pmatrix}
$$
Subtract the first column from each column $2,\ldots, t$:
$$
\begin{pmatrix}
1 &0& \cdots & 0&0  &0& \cdots &0\\
u_1 & u_2-u_1 &\cdots & u_t-u_1 & u_1 & u_2 &\cdots &u_k \\
\vdots\\
u_1^{n-1} &u_2^{n-1}-u_1^{n-1}& \cdots & u_t^{n-1}-u_1^{n-1}&(n- 1) u_1^{n-1} &(n-1 ) u_2^{n-1}& \cdots &(n-1)  u_k^{n-1}\\
u_1^n &u_2^n-u_1^n & \cdots & u_t^n-u_1^n&nu_1^n &nu_2^n &\cdots & n u_k^n\\
u_1^{n+1} &u_2^{n+1}-u_1^{n+1}& \cdots & u_t^{n+1}-u_1^{n+1} &(n+1) u_1^{n+1} & (n+1) u_2^{n+1}& \cdots & (n+1) u_k^{n+1}\\
\vdots\\
u_1^{2n} &u_2^{2n}-u_1^{2n}& \cdots & u_t^{2n}-u_1^{2n}& 2n u_1^{2n} & 2n u_2^{2n}& \cdots & 2n  u_k^{2n}
\end{pmatrix}
$$
For  $j=2n,2n-1,\ldots, 2$, subtract  $u_1$ times row $j-1$  from row $j$:
$$
\begin{pmatrix}
1 &0& \cdots & 0&0  &0& \cdots &0\\
0 & v_2 &\cdots & v_t & u_1 & u_2 &\cdots &u_k \\
\vdots\\
0 &u_2^{n-2}v_2 & \cdots & u_t^{n-2}v_t& u_1^{n-1} &(n-2) u_2^{n-2}v_2+u_2^{n-1} &\cdots & (n-2) u_k^{n-2} v_k+u_k^{n-1}
\\
0 &u_2^{n-1}v_2 & \cdots & u_t^{n-1}v_t& u_1^n &(n-1) u_2^{n-1}v_2+u_2^{n} &\cdots & (n-1) u_k^{n-1} v_k+u_k^{n}
\\
0 &u_2^{n}v_2& \cdots & u_t^{n}v_t & u_1^{n+1} & n u_2^{n}v_2+u_2^{n+1}& \cdots & n u_k^{n}v_k+u_k^{n+1}\\
\vdots
\\
0 &u_2^{2n-1}v_2& \cdots & u_t^{2n-1}v_t&  u_1^{2n} & (2n-1) u_2^{2n-1}v_2 + u_2^{2n}& \cdots & (2n-1)  u_k^{2n-1}v_k + u_k^{2n}
\end{pmatrix}
$$
where $v_j = u_j - u_1$.
Divide column $t+1$ by $u_1$ and for $j=2,\ldots, t$, divide  column $j$ by $(u_j-u_1)$:
$$
\begin{pmatrix}
1 &0& \cdots & 0&0  &0& \cdots &0\\
0 & 1 &\cdots & 1 & 1 & u_2 &\cdots &u_k \\
\vdots\\
0 &u_2^{n-2} & \cdots & u_t^{n-2}& u_1^{n-2} &(n-2) u_2^{n-2}v_2+u_2^{n-1} &\cdots & (n-2) u_k^{n-2} v_k +u_k^{n-1}
\\
0 &u_2^{n-1} & \cdots & u_t^{n-1}& u_1^{n-1} &(n-1) u_2^{n-1}v_2+u_2^{n} &\cdots & (n-1) u_k^{n-1} v_k+u_k^{n}
\\
0 &u_2^{n}& \cdots & u_t^{n} & u_1^{n} & n u_2^{n}v_2+u_2^{n+1}& \cdots & n u_k^{n}v_k+u_k^{n+1}\\
\vdots
\\
0 &u_2^{2n-1}& \cdots & u_t^{2n-1}&  u_1^{2n-1} & (2n-1) u_2^{2n-1}v_2 + u_2^{2n}& \cdots & (2n-1)  u_k^{2n-1}v_k + u_k^{2n}
\end{pmatrix}
$$
For $j=2,\ldots, k$, subtract $u_j$ times column $j$ from column $t+j$:
$$
\begin{pmatrix}
1 &0& \cdots & 0&0  &0& \cdots &0\\
0 & 1 &\cdots & 1 & 1 & 0 &\cdots &0 \\
\vdots\\
0 &u_2^{n-2} & \cdots & u_t^{n-2}& u_1^{n-2} &(n-2) u_2^{n-2}v_2&\cdots & (n-2) u_k^{n-2} v_k 
\\
0 &u_2^{n-1} & \cdots & u_t^{n-1}& u_1^{n-1} &(n-1) u_2^{n-1}v_2 &\cdots & (n-1) u_k^{n-1} v_k
\\
0 &u_2^{n}& \cdots & u_t^{n} & u_1^{n} & n u_2^{n}v_2& \cdots & n u_k^{n}v_k\\
\vdots
\\
0 &u_2^{2n-1}& \cdots & u_t^{2n-1}&  u_1^{2n-1} & (2n-1) u_2^{2n-1}v_2 & \cdots & (2n-1)  u_k^{2n-1}v_k 
\end{pmatrix}
$$
For $j=2,\ldots, k$, divide column $t+j$ by $u_j - u_1$:
\begin{align*}
&\begin{pmatrix}
1 &0& \cdots & 0&0  &0& \cdots &0\\
0 & 1 &\cdots & 1 & 1 & 0 &\cdots &0 \\
0 &u_2  & \cdots & u_t & u_1 &u_2 &\cdots & u_k
\\
\vdots
\\
0 &u_2^{n-1}  & \cdots & u_t^{n-1} & u_1^{n-1} &(n-1) u_2^{n-1} &\cdots & (n-1) u_k^{n-1}
\\
0 &u_2^{n} & \cdots & u_t^{n}  & u_1^{n} & n u_2^{n} & \cdots & n u_k^{n} \\
\vdots
\\
0 &u_2^{2n-1} & \cdots & u_t^{2n-1}&  u_1^{2n-1} & (2n-1) u_2^{2n-1} & \cdots & (2n-1)  u_k^{2n-1}
\end{pmatrix}\\
&= \begin{pmatrix}
1 & 0\\
0 & V^{\tilde u}_{t,k-1,2n-1}
\end{pmatrix}
\end{align*}
where $\tilde u = (u_2,\ldots, u_t, u_1)$. By iterating this argument, we have that injectivity of $V^u_{t,k,2n}$ follows from injectivity of
$V^{u'}_{t,0,2n-k}
$ 
where $u' = (u_{k+t},\ldots,u_t, u_1, \ldots, u_k)$, which is injective since it is a Vandermonde matrix and $2n>k+t$.

%
%
%
%

\section{Proof of Theorem \ref{thm:main2}}\label{app-proof-thm-main}
We assume throughout that $y\not\in \Kk$, which is a set of zero measure by Theorem \ref{prop:K_measurezero}.

Suppose that $\Ee_y = \emptyset$. Then,
 $\norm{\eta_y}_\infty < 1$ and $\mes \equiv 0$ is a solution. By continuity of $\eta_y$, there exists $\epsilon>0$ such that for all $y\in B_\epsilon(y)$, $\norm{\eta_{y'}}_\infty <1$ and zero is a solution to $\Pp_{\lambda}(y')$. So, $\partial_y(\Phi \mes_{y}) = 0$.

Suppose that  $\Ee_y = \TT$, then $\eta_y \equiv 1$ or $\eta_y\equiv -1$. Assume that $\eta_y \equiv 1$ (the argument for $\eta_y \equiv -1$ is similar). Then, there exist $\Ee \in \TT^n$ and $\be \in \RR_{\geq 0}^n$ such that $y - \Phi_\Ee \be =\lambda \delta_1$. Since $y\not\in \Gg^{+}$, there exists a neighbourhood around $y$ such that for all $y'\in B_\epsilon(y)$, $y' - \Phi_{\Ee'} \be'  = \lambda \delta_1$ for some $\be',\Ee'\in \RR_{\geq 0}^n\times \TT^n$.  So, $\Phi^* \pa{y' - \Phi_{\Ee'} \be' }/\lambda \equiv 1$ and $\mes_{\be',\Ee'}$ is a solution to $\Pp_\la(y')$. Therefore, $\hat \mu(y') = y' -\lambda \delta_1$ and  $\mathrm{Tr}\pa{\partial_y \hat \mu(y) } = n$.

It remains to consider the case where $\Ee_y$ is a discrete point set. 
Given $y\in \RR^n$, there exists $\mes_{\be,\Aa}$ such that $\Phi_\Aa$ is injective, and let  $k=\abs{\Aa}$. Let $s_y = {\eta_y}{\restriction_\Aa}$.  Define the function
$$
F(\be,\Aa,y)\in \RR^k\times \RR^k \times\RR^n = \Gamma_\Aa^* \pa{\Phi_\Aa \be - y} + \lambda\binom{s_y}{0_k}\in \RR^{2k}
$$
where $\Gamma_\Aa = [\Phi_\Aa,\Phi'_{\Aa}]$. 
We have $\partial_y F = -\Gamma_\Aa^*$, and writing  $u \eqdef (\be,\Aa)$,
\begin{align*}
\partial_u F(\be,\Aa,y) &= \Gamma_\Aa^* \Gamma_\Aa \begin{pmatrix}
\Id & 0\\
0 & \diag(\beta)
\end{pmatrix} + \begin{pmatrix}
\Id & 0\\
0 & \diag(z)
\end{pmatrix} \\
&= \pa{\Gamma_\Aa^* \Gamma_\Aa  + \begin{pmatrix}
\Id & 0\\
0 & \diag((Z_i/\be_i)_i)
\end{pmatrix}  } \begin{pmatrix}
\Id & 0\\
0 & \diag(\be)
\end{pmatrix}
\end{align*}
where $Z = \pa{\dotp{\Phi_\Aa \be - y}{\phi''(\cdot)}}_{\restriction_\Aa}\in\RR^k$. Since $\Gamma_\Aa^* \Gamma_\Aa $ is invertible (by Appendix \ref{app:fullrank}),  
we can apply the implicit function theorem to define a function $g$ in a small neighbourhood $U$ around $y$, such that $y'\in U\mapsto (\be', \Aa')$ is a $\Cc^1$ function. If we can show that $\mes_{\be',\Aa'}$ is indeed a solution of $\Pp_\la(y')$, then this allows us to apply Theorem \ref{prop:div_formula} to  compute the DOF.

Let $m\eqdef \abs{\Ee_y}$ and write $M^y\eqdef \Phi_{\Ee_y}$.  Let $J$ be such that $(\Ee_{y})_J = \Aa$, and let $\beta \in \RR^m$ be such that $\mathrm{Supp}(\beta) = J$ and $\mes_{\beta,\Ee_y}$ solves \eqref{eq:blasso}.  Note that $M_J^y= \Phi_{\Aa}$ and recall that $M^y$ is full rank due to Lemma \ref{lem:fourier_inj}.
 
\paragraph{Properties of $\beta$:} 
 By Lemma \ref{lem_sol}, since $\ker(M^y) = \ens{0}$, the solution to \eqref{eq:blasso} is unique and equal to  $\mes_{\beta,\Ee_y}$ where
$$
\beta = (M^y)^\dagger (y - ((M^y)^*)^\dagger \lambda \sigma), \qwhereq \sigma= (\eta_y)_{\restriction_{\Ee_y}}.
$$
Letting  $I \eqdef [m]\setminus J$, we have that 
$$
\beta_I = \pa{(M^y)^\dagger (y - ((M^y)^*)^\dagger \lambda \sigma)}_I = 0_I.
$$

Write $\Ee_y = \{x_i\}_{i=1}^m$, and for each $i$, let $\bell_i$ be the first integer such that $\eta^{2\bell_i}(x_i) \neq 0$. By definition, $y\in \Qq_{m,\sigma,I,\bell}$.

\paragraph{Constructing a solution $\beta'$ for $\Pp_\lambda(y')$:}

Since $y\not\in \Hh$, we have $y$ is in the interior of $\Pi_Y(\Qq_{m,\sigma,I,\bell})$ and so, there exists
 $\epsilon>0$ such that for all $y'\in \Bb_\epsilon(y)$: there exists $\abs{\Ee'} = m$ and $M^{y'}= \Phi_{\Ee'}$ such that 
\begin{equation}\label{eq:suppI}
\pa{ (M^{y'})^\dagger (y' - ((M^{y'})^*)^\dagger \lambda \sigma)
}_I = 0_I,
\end{equation}
and, we can write $\Ee' = \{x_i'\}_{i=1}^m$ so that for each $i$, $(\eta^{y'})^{2\ell}(x_i) = 0$ for all $\ell<\bell_i$. By definition, $\Ee'$ is contained in the extended support of $y'$. By Proposition \ref{prop:cont_ext},  $\Ee'$ is precisely the extended suport with $\Ee_{y'} = \Ee'$ such that $y'\in \Bb_\epsilon(y) \mapsto \Ee_{y'}$ is a continuous function.
So, $M^{y'} \to M^y$ and $M^{y'}_S$ is injective, and since rank is preserved,  $((M^{y'})^*)^\dagger \to ((M^{y})^*)^\dagger$.

Define
$$
\beta' \eqdef (M^{y'})^\dagger f(y'), \qwhereq f(y') = y' - ((M^{y'})^*)^\dagger \lambda\sigma.
$$
By \eqref{eq:suppI}, $\beta_I' = 0$.
Note that $f$ is continuous as $y'$ changes, so since $\beta_J$ has all non-zero entries, $\sign(\beta'_J ) = \sign(\beta_J)$ when $y'$ is sufficiently close to $y$.

\section{Proof of Lemma \ref{lem_pos1}}\label{app:pos_lem}

Lemmas \ref{lem_pos1}  follows from the Fenchel-Rockafellar duality theorem,  which states that given proper, convex, lsc functionals $E$ and $F$, denoting the convex conjugates by $E^*$ and $F^*$, the dual of
\begin{equation}\label{eq:gen}
\inf_{\mes\in\Mm(\Om)} E(\Phi \mes) + F(\mes). \tag{$\Pp$}
\end{equation}
is
\begin{equation}\label{eq:gendual}
\sup_p -E^*(p)- F^*(-\Phi^*p).  \tag{$\Dd$}
\end{equation}
Moreover, if there exists $ \mes \in \mathrm{dom}(F)$ and $E$ is continuous at $\Phi  \mes$, then we have strong duality \eqref{eq:gen} = \eqref{eq:gendual}, there exists a dual solution, given primal and dual solutions $\mes_*$ and $p_*$, we have
$$
\Phi \mes_* \in \partial E^*(p_*) \qandq -\Phi^*p_* \in \partial F(\mes_*).
$$

For Lemma \ref{lem_pos1},
we can write \eqref{eq:pos_constr1} as \eqref{eq:gen}
with $E(z) \eqdef \frac{1}{2} \norm{z -y}_2^2$ and $F \eqdef   \iota_{\enscond{\mes}{\mes\geq 0}}$ which are proper, convex, lower semicontinuous functionals. Their convex conjugates are $E^*:\RR^n \to \RR$ and $F^*:\Cc(\Om) \to \RR$
$$
E^*(p) = \frac{1}{2}\norm{p}^2 + \dotp{p}{y} \qandq  F^* = \iota_{\enscond{f}{ f\leq 0}}
$$
Note that for $\mes \equiv 0$, $F(\mes) = 0 <\infty$ and clearly, $E$ is continuous at $\Phi \mes$. So, by Fenchel-Rockafellar duality, we have strong duality between \eqref{eq:pos_constr1} and \eqref{eq:dual_pos_constr1}. Moreover, any primal and dual solutions satisfy
$$
-\Phi^* p_* \in \partial F(\mes_*) = \enscond{f\in\Cc(\Om)}{f\leq 0, \;  f(x) =  0, \; \forall x\in \mathrm{Supp}(\mes_*)}\qandq \Phi \mes_* \in \partial E^*(p_*) = p_*+y
$$ 
and hence, the stated the primal dual relations hold.


\bibliographystyle{siam}
\bibliography{biblio}

\begin{thebibliography}{10}

\bibitem{akaike1998information}
{\sc H.~Akaike}, {\em Information theory and an extension of the maximum
  likelihood principle}, in Selected papers of hirotugu akaike, Springer, 1998,
  pp.~199--213.

\bibitem{azais-spike2014}
{\sc J.-M. Azais, Y.~De~Castro, and F.~Gamboa}, {\em Spike detection from
  inaccurate samplings}, Applied and Computational Harmonic Analysis, 38
  (2015), pp.~177--195.

\bibitem{bach2017breaking}
{\sc F.~Bach}, {\em Breaking the curse of dimensionality with convex neural
  networks}, The Journal of Machine Learning Research, 18 (2017), pp.~629--681.

\bibitem{BonnansShapiro2000}
{\sc J.~Bonnans and A.~Shapiro}, {\em Perturbation analysis of optimization
  problems}, Springer Series in Operations Research, Springer-Verlag, New York,
  2000.

\bibitem{boyd2017alternating}
{\sc N.~Boyd, G.~Schiebinger, and B.~Recht}, {\em The alternating descent
  conditional gradient method for sparse inverse problems}, SIAM Journal on
  Optimization, 27 (2017), pp.~616--639.

\bibitem{boyer2019representer}
{\sc C.~Boyer, A.~Chambolle, Y.~D. Castro, V.~Duval, F.~De~Gournay, and
  P.~Weiss}, {\em On representer theorems and convex regularization}, SIAM
  Journal on Optimization, 29 (2019), pp.~1260--1281.

\bibitem{bredies-inverse2013}
{\sc K.~Bredies and H.~K. Pikkarainen}, {\em Inverse problems in spaces of
  measures}, ESAIM: Control, Optimisation and Calculus of Variations, 19
  (2013), pp.~190--218.

\bibitem{candes-superresolution2013}
{\sc E.~J. Cand{\`e}s and C.~Fernandez-Granda}, {\em Super-resolution from
  noisy data}, Journal of Fourier Analysis and Applications, 19 (2013),
  pp.~1229--1254.

\bibitem{candes-towards2013}
\leavevmode\vrule height 2pt depth -1.6pt width 23pt, {\em Towards a
  mathematical theory of super-resolution}, Communications on Pure and Applied
  Mathematics, 67 (2014), pp.~906--956.

\bibitem{candes2013unbiased}
{\sc E.~J. Candes, C.~A. Sing-Long, and J.~D. Trzasko}, {\em Unbiased risk
  estimates for singular value thresholding and spectral estimators}, IEEE
  transactions on signal processing, 61 (2013), pp.~4643--4657.

\bibitem{chizat2018global}
{\sc L.~Chizat and F.~Bach}, {\em On the global convergence of gradient descent
  for over-parameterized models using optimal transport}, in Advances in neural
  information processing systems, 2018, pp.~3036--3046.

\bibitem{coste1999introduction}
{\sc M.~COSTE}, {\em An introduction to o-minimal geometry},  (1999).

\bibitem{coste2000introduction}
{\sc M.~Coste}, {\em An introduction to semialgebraic geometry}, Citeseer,
  2000.

\bibitem{deCastro-exact2012}
{\sc Y.~De~Castro and F.~Gamboa}, {\em Exact reconstruction using {Beurling}
  minimal extrapolation}, Journal of Mathematical Analysis and applications,
  395 (2012), pp.~336--354.

\bibitem{de2016exact}
{\sc Y.~De~Castro, F.~Gamboa, D.~Henrion, and J.-B. Lasserre}, {\em Exact
  solutions to super resolution on semi-algebraic domains in higher
  dimensions}, IEEE Transactions on Information Theory, 63 (2016),
  pp.~621--630.

\bibitem{deledalle2014stein}
{\sc C.-A. Deledalle, S.~Vaiter, J.~Fadili, and G.~Peyr{\'e}}, {\em Stein
  unbiased gradient estimator of the risk (sugar) for multiple parameter
  selection}, SIAM Journal on Imaging Sciences, 7 (2014), pp.~2448--2487.

\bibitem{denoyelle2019sliding}
{\sc Q.~Denoyelle, V.~Duval, G.~Peyr{\'e}, and E.~Soubies}, {\em The sliding
  frank-wolfe algorithm and its application to super-resolution microscopy},
  Inverse Problems,  (2019).

\bibitem{donoho1995adapting}
{\sc D.~L. Donoho and I.~M. Johnstone}, {\em Adapting to unknown smoothness via
  wavelet shrinkage}, Journal of the american statistical association, 90
  (1995), pp.~1200--1224.

\bibitem{dossal2013degrees}
{\sc C.~Dossal, M.~Kachour, M.~Fadili, G.~Peyr{\'e}, and C.~Chesneau}, {\em The
  degrees of freedom of the {Lasso} for general design matrix}, Statistica
  Sinica,  (2013), pp.~809--828.

\bibitem{duval2015exact}
{\sc V.~Duval and G.~Peyr{\'e}}, {\em Exact support recovery for sparse spikes
  deconvolution}, Foundations of Computational Mathematics, 15 (2015),
  pp.~1315--1355.

\bibitem{duval2017sparse}
\leavevmode\vrule height 2pt depth -1.6pt width 23pt, {\em Sparse
  regularization on thin grids i: the lasso}, Inverse Problems, 33 (2017),
  p.~055008.

\bibitem{eldar-gsure}
{\sc Y.~C. Eldar}, {\em Generalized {SURE} for exponential families:
  Applications to regularization}, IEEE Transactions on Signal Processing, 57
  (2009), pp.~471--481.

\bibitem{fisher1975spline}
{\sc S.~Fisher and J.~W. Jerome}, {\em Spline solutions to l1 extremal problems
  in one and several variables}, Journal of Approximation Theory, 13 (1975),
  pp.~73--83.

\bibitem{giryes2011projected}
{\sc R.~Giryes, M.~Elad, and Y.~C. Eldar}, {\em The projected gsure for
  automatic parameter tuning in iterative shrinkage methods}, Applied and
  Computational Harmonic Analysis, 30 (2011), pp.~407--422.

\bibitem{golub1979generalized}
{\sc G.~H. Golub, M.~Heath, and G.~Wahba}, {\em Generalized cross-validation as
  a method for choosing a good ridge parameter}, Technometrics, 21 (1979),
  pp.~215--223.

\bibitem{hudson1978nie}
{\sc H.~Hudson}, {\em A natural identity for exponential families with
  applications in multiparameter estimation}, Annals of Statistics, 6 (1978),
  pp.~473--484.

\bibitem{Hwang82}
{\sc J.~T. Hwang}, {\em Improving upon standard estimators in discrete
  exponential families with applications to poisson and negative binomial
  cases}, Annals of Statistics, 10 (1982), pp.~857--867.

\bibitem{kato2009degrees}
{\sc K.~Kato}, {\em On the degrees of freedom in shrinkage estimation}, Journal
  of Multivariate Analysis, 100 (2009), pp.~1338--1352.

\bibitem{mallows1973some}
{\sc C.~L. Mallows}, {\em Some comments on c p}, Technometrics, 15 (1973),
  pp.~661--675.

\bibitem{MeyerWoodroofe}
{\sc M.~Meyer and M.~Woodroofe}, {\em On the degrees of freedom in
  shape-restricted regression}, Annals of Statistics, 28 (2000),
  pp.~1083--1104.

\bibitem{ramani2008monte}
{\sc S.~Ramani, T.~Blu, and M.~Unser}, {\em Monte-carlo sure: A black-box
  optimization of regularization parameters for general denoising algorithms},
  IEEE Transactions on image processing, 17 (2008), pp.~1540--1554.

\bibitem{ramani2012regularization}
{\sc S.~Ramani, Z.~Liu, J.~Rosen, J.-F. Nielsen, and J.~A. Fessler}, {\em
  Regularization parameter selection for nonlinear iterative image restoration
  and mri reconstruction using gcv and sure-based methods}, IEEE Transactions
  on Image Processing, 21 (2012), pp.~3659--3672.

\bibitem{rosset2004boosting}
{\sc S.~Rosset, J.~Zhu, and T.~Hastie}, {\em Boosting as a regularized path to
  a maximum margin classifier}, Journal of Machine Learning Research, 5 (2004),
  pp.~941--973.

\bibitem{schwarz1978estimating}
{\sc G.~Schwarz et~al.}, {\em Estimating the dimension of a model}, The annals
  of statistics, 6 (1978), pp.~461--464.

\bibitem{stein1981estimation}
{\sc C.~M. Stein}, {\em Estimation of the mean of a multivariate normal
  distribution}, The annals of Statistics,  (1981), pp.~1135--1151.

\bibitem{tibshirani2012degrees}
{\sc R.~J. Tibshirani and J.~Taylor}, {\em Degrees of freedom in lasso
  problems}, The Annals of Statistics, 40 (2012), pp.~1198--1232.

\bibitem{unser2017splines}
{\sc M.~Unser, J.~Fageot, and J.~P. Ward}, {\em Splines are universal solutions
  of linear inverse problems with generalized tv regularization}, SIAM Review,
  59 (2017), pp.~769--793.

\bibitem{vaiter2017degrees}
{\sc S.~Vaiter, C.~Deledalle, J.~Fadili, G.~Peyr{\'e}, and C.~Dossal}, {\em The
  degrees of freedom of partly smooth regularizers}, Annals of the Institute of
  Statistical Mathematics, 69 (2017), pp.~791--832.

\bibitem{vaiter-local-behavior}
{\sc S.~Vaiter, C.~Deledalle, G.~Peyr{\'e}, C.~Dossal, and M.~J. Fadili}, {\em
  Local behavior of sparse analysis regularization: Applications to risk
  estimation}, Applied and Computational Harmonic Analysis, 35 (2013),
  pp.~433--451.

\bibitem{wilkie1996model}
{\sc A.~J. Wilkie}, {\em Model completeness results for expansions of the
  ordered field of real numbers by restricted pfaffian functions and the
  exponential function}, Journal of the American Mathematical Society, 9
  (1996), pp.~1051--1094.

\bibitem{zou2007degrees}
{\sc H.~Zou, T.~Hastie, R.~Tibshirani, et~al.}, {\em On the degrees of freedom
  of the lasso}, The Annals of Statistics, 35 (2007), pp.~2173--2192.

\end{thebibliography}

\end{document}